\documentclass[10pt,twocolumn,letterpaper]{article}

\usepackage[pagenumbers]{cvpr} 

\usepackage{graphicx}
\usepackage{amsmath}
\usepackage{amssymb}
\usepackage{booktabs}

\usepackage{amsmath,amsfonts,bm}

\def\eqref#1{equation~\ref{#1}}

\def\1{\bm{1}}

\DeclareMathAlphabet{\mathsfit}{\encodingdefault}{\sfdefault}{m}{sl}
\SetMathAlphabet{\mathsfit}{bold}{\encodingdefault}{\sfdefault}{bx}{n}

\def\gF{{\mathcal{F}}}

\def\gL{{\mathcal{L}}}

\def\gS{{\mathcal{S}}}

\def\gW{{\mathcal{W}}}
\def\gX{{\mathcal{X}}}
\def\gY{{\mathcal{Y}}}

\def\sR{{\mathbb{R}}}

\newcommand{\E}{\mathbb{E}}

\DeclareMathOperator*{\arginf}{arg\,inf}

\usepackage[numbers]{natbib}
\usepackage{url}
\usepackage{cite}
\usepackage{graphicx}
\usepackage{booktabs}

\usepackage{framed}
\usepackage{amssymb}
\usepackage{amsfonts}
\usepackage{mathrsfs}
\usepackage{mathtools}
\usepackage{array}
\usepackage{amsthm}
\usepackage{verbatim} 
\usepackage{enumerate}
\usepackage{bbm}
\usepackage{commath}
\usepackage{wrapfig}
\usepackage{amsbsy}
\usepackage{float}
\usepackage{algorithm}
\usepackage{tabularx}
\usepackage{listings}
\usepackage{xcolor}
\usepackage{enumitem}
\usepackage{multirow}
\usepackage{colortbl}

\newlength\savewidth
\newlength\thinwidth

\definecolor{Gray}{gray}{0.93}
\newcolumntype{a}{>{\columncolor{Gray}}c}
\definecolor{LightCyan}{rgb}{0.88,1,1}
\definecolor{highlightRowColor}{gray}{0.93}
\definecolor{HighlightBlue}{RGB}{230, 235, 247}

\newcommand{\HC}[1]{\ifthenelse{\isodd{#1}}{\rowcolor{highlightRowColor}}{\rowcolor{white}}}

\definecolor{codegreen}{rgb}{0,0.3,0.6}
\definecolor{codegray}{rgb}{0.5,0.5,0.5}
\definecolor{codepurple}{rgb}{0.58,0,0.82}
\definecolor{backcolour}{rgb}{0.95,0.95,0.92}

\lstdefinestyle{mystyle}{
    basicstyle=\tiny,
    commentstyle=\color{codegreen},
    keywordstyle=\color{magenta},
    numberstyle=\tiny\color{codegray},
    stringstyle=\color{codepurple},
    basicstyle=\fontsize{8.5}{9}\selectfont\ttfamily,
    breakatwhitespace=false,         
    breaklines=true,                 
    captionpos=b,                    
    keepspaces=true,                 
    numbers=left,                    
    numbersep=5pt,                  
    showspaces=false,                
    showstringspaces=false,
    frame = single
}

\newtheorem{thm}{Theorem}

\newtheorem{theorem}[thm]{Theorem}
\newtheorem{lemma}[thm]{Lemma}

\newtheorem{corollary}[thm]{Corollary}

\newcommand{\Lip}{\textnormal{Lip}}

\usepackage[pagebackref,breaklinks,colorlinks]{hyperref}

\usepackage[capitalize]{cleveref}
\crefname{section}{Sec.}{Secs.}
\Crefname{section}{Section}{Sections}
\Crefname{table}{Table}{Tables}
\crefname{table}{Tab.}{Tabs.}

\newcommand{\gray}[1]{{\color{gray}{#1}}}

\def\cvprPaperID{***} 
\def\confName{CVPR}
\def\confYear{2022}

\begin{document}

\title{Robust Contrastive Learning against Noisy Views}

\author{Ching-Yao Chuang$^\dagger$ \; R Devon Hjelm$^\ddag$ \; Xin Wang$^\ddag$ \; Vibhav Vineet$^\ddag$ 
\\
Neel Joshi$^\ddag$ \; Antonio Torralba$^\dagger$ \; Stefanie Jegelka$^\dagger$ \; Yale Song$^\ddag$\\
$^\dagger$MIT CSAIL \; $^\ddag$Microsoft Research\\
{\small \url{https://github.com/chingyaoc/RINCE}}
}
\maketitle

\begin{abstract}
\vspace{-.5em}
Contrastive learning relies on an assumption that positive pairs contain related views, e.g., patches of an image or co-occurring multimodal signals of a video, that share certain underlying information about an instance. But what if this assumption is violated? The literature suggests that contrastive learning produces suboptimal representations in the presence of noisy views, e.g., false positive pairs with no apparent shared information. In this work, we propose a new contrastive loss function that is robust against noisy views. We provide rigorous theoretical justifications by showing connections to robust symmetric losses for noisy binary classification and by establishing a new contrastive bound for mutual information maximization based on the Wasserstein distance measure. The proposed loss is completely modality-agnostic and a simple drop-in replacement for the InfoNCE loss, which makes it easy to apply to existing contrastive frameworks. We show that our approach provides consistent improvements over the state-of-the-art on image, video, and graph contrastive learning benchmarks that exhibit a variety of real-world noise patterns.
\end{abstract}
\vspace{-3mm}

\section{Introduction}
Contrastive learning~\cite{chopra2005learning, hadsell2006dimensionality, oord2018representation} has become one of the most prominent self-supervised approaches to learn representations of high-dimensional signals, producing impressive results with image~\cite{tian2020contrastive, he2020momentum, chen2020simple, chen2020improved, grill2020bootstrap, zbontar2021barlow}, text~\cite{logeswaran2018efficient, chuang2020debiased, robinson2020contrastive, giorgi2020declutr}, audio~\cite{baevski2020wav2vec, saeed2021contrastive, wang2021multi}, and video~\cite{miech2020end, ma2020active, morgado2021audio}. The central idea is to learn representations that capture the underlying information shared between different ``views'' of data~\cite{oord2018representation, tian2020makes}. For images, the views are typically constructed by applying common data augmentation techniques, such as jittering, cropping, resizing and rotation~\cite{chen2020simple}, and for video the views are often chosen as adjacent frames~\cite{sermanet2018time} or co-occurring multimodal signals, such as video and the corresponding optical flow~\cite{han2020self}, audio~\cite{morgado2021audio} and transcribed speech~\cite{miech2020end}.

\begin{figure}[h]
\vspace{-1mm}
\begin{center}
\resizebox{\columnwidth}{!}{
\includegraphics[width=0.8\linewidth]{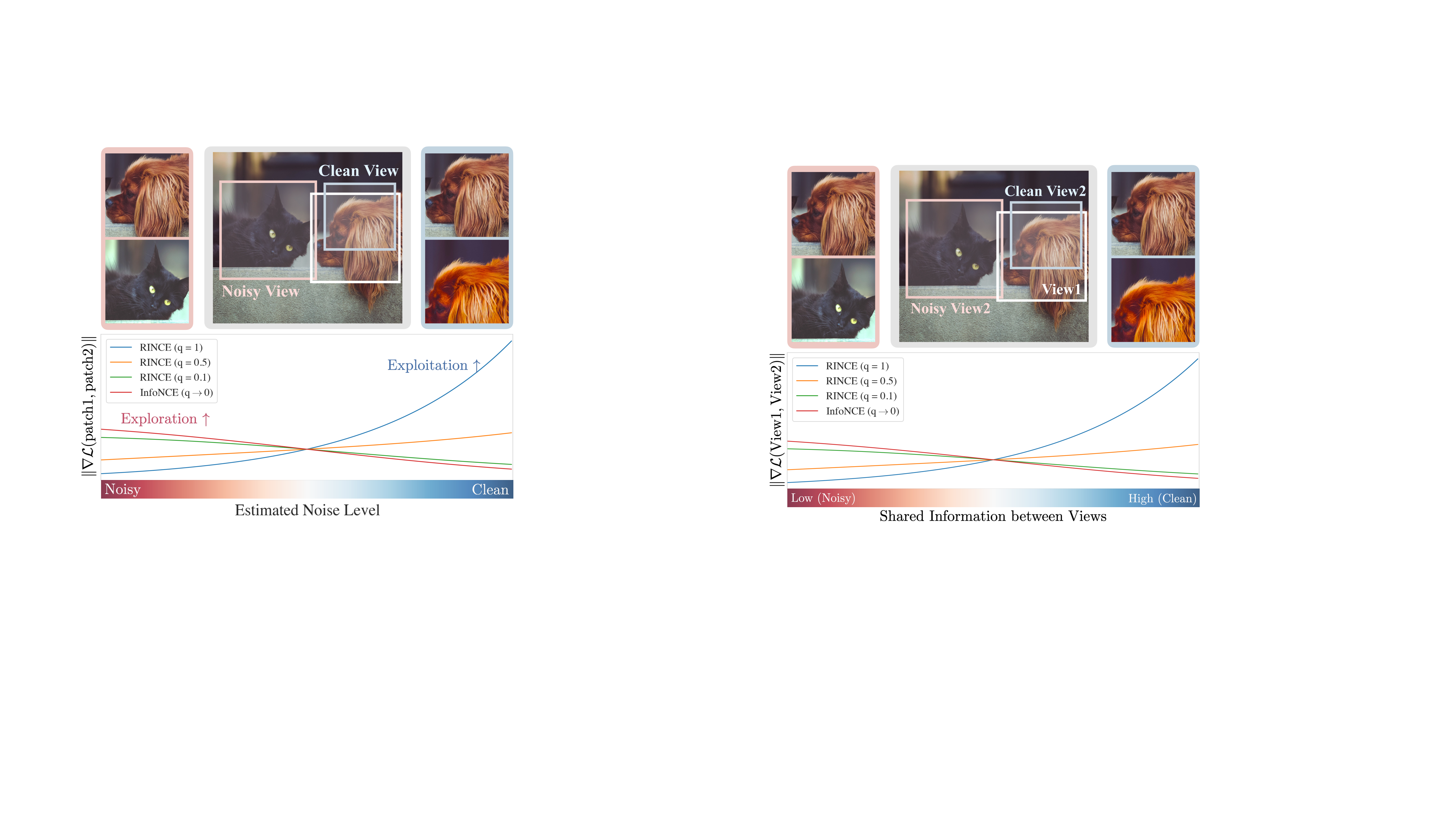}
}
\end{center}
\vspace{-3mm}
\caption{\textbf{Noisy views can deteriorate contrastive learning.} We propose a new contrastive loss function (RINCE) that rescales the sample importance in the gradient space based on an estimated noise level. With a simple turn of a knob ($q\in(0,1]$), we can upweight or downweight sample pairs with low shared information.}
\vspace{-3mm}
\label{fig_1}
\end{figure}

Designing the \emph{right} contrasting views has shown to be a key ingredient of contrastive learning~\cite{chen2020simple,patrick2021compositions}. This often requires domain knowledge, intuition, trial-and-error (and luck!). What would happen if the views are wrongly chosen and do not provide meaningful shared information? Prior work has reported deteriorating effects of such \emph{noisy} views in contrastive learning under various scenarios, e.g., unrelated image patches due to extreme augmentation~\citep{tian2020makes}, irrelevant video-audio pairs due to overdubbing~\citep{morgado2021robust}, and misaligned video-caption pairs~\citep{miech2020end}. The major issue with noisy views is that representations of different views are forced to align with each other even if there is no meaningful shared information. This often leads to suboptimal representations that merely capture spurious correlations~\cite{arjovsky2019invariant} or make them collapse to a trivial solution~\cite{jing2021understanding}. Worse yet, when we attempt to learn from large-scale unlabeled data -- i.e., the scenario where self-supervised learning is particularly expected to shine -- the issue is only aggravated because of the increased noise in the real-world data~\cite{lee2021acav100m}, hindering the ultimate success of contrastive learning.

Consequently, a few attempts have been made to design contrastive approaches that are noise-tolerant. For example, \citet{morgado2021robust} optimize a soft instance discrimination loss to weaken the impact of noisy views. \citet{miech2020end} address the misalignment between video and captions by aligning multiple neighboring segments of a video. However, existing approaches are often tied to specific modalities or make assumptions that may not hold for general scenarios, e.g., MIL-NCE~\cite{miech2020end} is not designed to address the issues of irrelevant audio-visual signals.

In this work, we develop a principled approach to make contrastive learning robust against noisy views. We start by making connections between contrastive learning and the classical noisy binary classification in supervised learning~\cite{natarajan2013learning, ghosh2015making}. This allows us to explore the wealth of literature on learning with noisy labels~\cite{ghosh2017robust, li2017learning, veit2017learning}. In particular, we focus on a family of robust loss functions that has the \emph{symmetric} property~\cite{ghosh2015making}, which provides strong theoretical guarantees against noisy labels in binary classification. We then show a functional form of contrastive learning that can satisfy the symmetry condition if given a proper symmetric loss function, motivating the design of new contrastive loss functions that provide similar theoretical guarantees. 

This leads us to propose \textbf{R}obust \textbf{I}nfo\textbf{NCE} (RINCE), a contrastive loss function that satisfies the symmetry condition. RINCE can be understood as a generalized form of the contrastive objective that is robust against noisy views. Intuitively, its symmetric property provides an implicit means to reweight sample importance in the gradient space \emph{without requiring} an explicit form of noise estimator. It also provides a simple ``knob'' (a real-valued scalar $q\in(0,1]$) that controls the behavior of the loss function balancing the exploration-exploitation trade-off (i.e., from being conservative to playing adventures on potentially noisy samples). 

We also provide a theoretical analysis of the proposed RINCE objective and show that it extends the analyses by \citet{ghosh2015making} to the self-supervised contrastive learning regime. Furthermore, we relate the proposed loss function to dependency measurement. Analogous to InfoNCE loss, which is a lower bound of mutual information between two views \citep{oord2018representation}, we show that RINCE is a lower bound of Wassersein Dependency Measure (WDM) \citep{ozair2019wasserstein} even in the noisy setting. By replacing the KL divergence in the mutual information estimator with the Wasserstein distance, WDM is able to capture the geometry of the representation space via the equipped metric space and provides robustness against noisy views better than the KL divergence, both in theory and practice. In particular, the features learned with RINCE achieve better class-wise separation, which is proved to be crucial to improve generalization \citep{chuang2021measuring}.

Despite its rigorous theoretical background, implementing RINCE requires only a few lines code and can be a simple drop-in replacement for the InfoNCE loss to make contrastive learning robust against noisy views. Since InfoNCE sets the basis for many modern contrastive methods such as SimCLR~\cite{chen2020simple} and MoCo-v1/v2/v3~\cite{he2020momentum, chen2020improved, chen2021empirical}, our construction can be easily applied to many existing frameworks.

Finally, we provide strong empirical evidence demonstrating the robustness of RINCE against noisy views under various scenarios with different modalities and noise types. We show that RINCE improves over the state-of-the-art in image~\citep{krizhevsky2009learning, deng2009imagenet}, video~\citep{lee2021acav100m,kay2017kinetics} and graph \citep{morris2020tudataset} self-supervised learning benchmarks, demonstrating its generalizability across multiple modalities. We also show that RINCE exhibits strong robustness against different types of noise such augmentation noise~\cite{tian2020makes,you2020graph}, label noise~\cite{natarajan2013learning, khosla2020supervised}, and noisy audio-visual correspondence~\citep{morgado2021robust}. The improvement is consistently observed across different dataset scales and training epochs, demonstrating the scalability and computational efficiency. In short, our main contributions are:
\begin{itemize}[leftmargin=0.5cm]\setlength{\itemsep}{-1pt}
\item We propose RINCE, a new contrastive learning objective that is robust against noisy views of data;
\item We provide a theoretical analysis to relate the proposed loss to symmetric losses and dependency measurement;
\item We demonstrate our approach on real-world scenarios of image, video, and graph contrastive learning.
\end{itemize}

\section{Related Work}

\paragraph{Contrastive Learning}
Contrastive approaches have become prominent in unsupervised representation learning~\cite{chopra2005learning, hadsell2006dimensionality, hjelm2018learning}: InfoNCE~\cite{oord2018representation} and its variants~\cite{he2020momentum, chen2020simple, grill2020bootstrap, zbontar2021barlow} achieve state-of-the-art across different modalities~\citep{logeswaran2018efficient, baevski2020wav2vec, miech2020end, ma2020active, morgado2021audio}. Modern approaches improve upon InfoNCE from different directions. One line of work focuses on modifying training mechanisms, e.g., appending projection head~\citep{chen2020simple}, momentum encoder with dynamic dictionary update~\citep{he2020momentum, chen2020improved}, siamese networks with stop gradient trick~\citep{chen2021exploring, grill2020bootstrap}, and online cluster assignment~\citep{caron2020unsupervised}. Another line of work refines the loss function itself to make it more effective, e.g., upweight hard negatives \citep{kalantidis2020hard, robinson2020contrastive}, correct false negatives~\citep{chuang2020debiased}, and alleviate feature suppression~\citep{robinson2021can}. Along this second line of work, we propose a new contrastive loss function robust against noisy views. Some of prior work in this direction~\cite{morgado2021robust,miech2020end} was demonstrated on limited modalities only; we demonstrate its generality on image \citep{deng2009imagenet}, video \citep{lee2021acav100m,kay2017kinetics}, and graph \citep{velivckovic2018deep, hassani2020contrastive, you2020graph} contrastive learning scenarios. Our approach is orthogonal to the first line of work; our loss function can easily be applied to some of existing training mechanisms such as SimCLR~\cite{chen2020simple} and MoCo-v1/v2/v3~\cite{he2020momentum, chen2020improved, chen2021empirical}.

\vspace{-1em}\paragraph{Robust Loss against Noisy Labels}
Learning with noisy labels has been actively explored in recent years \citep{natarajan2013learning, sukhbaatar2014training, ghosh2015making, xiao2015learning, liu2015classification, patrini2017making, li2017learning, veit2017learning, jiang2018mentornet, ren2018learning,han2018co}. One line of work attempts to develop robust loss functions that are noise-tolerant \citep{ghosh2015making, ghosh2017robust, zhang2018generalized, wang2019symmetric}. \citet{ghosh2015making} prove that symmetric loss functions are robust against noisy labels, e.g., Mean Absolute Error (MAE) \citep{ghosh2017robust}, while commonly used Cross Entropy (CE) loss is not. Based on this idea, \citet{zhang2018generalized} propose the generalized cross entropy loss to combine MAE and CE loss functions. A similar idea is adopted in \citep{wang2019symmetric} by combining the reversed cross entropy loss with CE loss. In the next section, we relate noisy views to noisy labels by interpreting contrastive learning as binary classification, and developed a robust symmetric contrastive loss that enjoys the similar theoretical guarantees.

\section{Prelim: From Noisy Labels to Noisy Views}

We start by connecting two seemingly different but related frameworks: supervised binary classification with noisy labels and self-supervised contrastive learning with noisy views. We then introduce a family of symmetric loss functions that is noise-tolerant and show how we can transform contrastive objectives to a symmetric form.

\subsection{Symmetric Losses for Noisy Labels}

Denoting the input space by $\gX$ and the binary output space by $\gY = \{-1, 1\}$, let 
$\gS = \{x_i, y_i\}_{i=1}^m$ be the unobserved clean dataset that is drawn i.i.d. from the data distribution $\mathcal{D}$. In the noisy setting, the learner obtains a noisy dataset $\gS_\eta = \{x_i, \hat{y}_i\}_{i=1}^m$, where $\hat{y}_i = y_i$ with probability $1 - \eta_{x_i}$ and $\hat{y}_i = -y_i$ with probability $\eta_{x_i}$. Note that the noise rate $\eta_x$ is data point-dependent. For a classifier $f \in \gF: \gX \rightarrow \sR$, the expected risk under the noise-free scenario is $R_{\ell}(f) = \E_\mathcal{D}[\ell(f(x), y)]$ where $\ell: \sR \times \gY \rightarrow \sR$ is a binary classification loss function. When the noise exists, the learner minimizes the noisy expected risk $R_{\ell}^\eta(f) = \E_{\mathcal{D}_\eta}[\ell(f(x), \hat{y})]$. 

\citet{ghosh2015making} show that \emph{symmetric} loss functions are robust against noisy labels in binary classification. In particular, a loss function $\ell$ is symmetric if it sums to a constant:
\begin{align}
    \label{eq_sym}
    \ell(s, 1) + \ell(s, -1) = c, \;\;\;\;\forall s \in \sR,
\end{align}
where $s$ is the prediction score from $f$. Note that the symmetry condition should also hold with the gradients w.r.t. $s$. They show that if the noise rate is $\eta_x \leq \eta_{\max} < 0.5, \forall x \in \gX$ and if the loss is symmetric and non-negative, the minimizer of the noisy risk $f_\eta^\ast = \arginf_{f \in \gF} R^\eta(f)$ approximately minimizes the clean risk:
\begin{align*}
    R(f_\eta^\ast) \leq \epsilon / (1 - 2 \eta_{\max}),
\end{align*}
where $\epsilon = \inf_{f \in \gF} R(f)$ is the optimal clean risk. This implies that the noisy risk under symmetric loss is a good surrogate of the clean risk. In Appendix \ref{sec_exp_bound}, we further relax the non-negative constraint on the loss with a corollary.\footnote{This is important for our proposed RINCE loss that involves an exponential function $\ell(s,y) = -ye^s$, which can produce negative values.}

\subsection{Towards Symmetric Contrastive Objectives}
The results above suggest that we can achieve robustness against noisy views if a contrastive objective can be expressed in a form that satisfies the symmetry condition in the binary classification framework. To this end, we first relate contrastive learning to binary classification, and then express it in a form where symmetry can be achieved.

\vspace{-1em}\paragraph{Contrastive learning as binary classification.} 
Given two views $X$ and $V$, we can interpret contrastive learning as noisy binary classification operating over pairs of samples $(x, v)$ with a label $1$ if it is sampled from the joint distribution, $(x, v) \sim P_{XV}$, and $-1$ if it comes from the product of marginals, $(x, v') \sim P_{X}P_{V}$. In the presence of noisy views, some negative pairs $(x, v') \sim P_{X}P_{V}$ could be mislabeled as positive, introducing noisy labels.

To see this more concretely, let us consider the InfoNCE loss~\cite{oord2018representation}, one of the most widely adopted contrastive objectives~\cite{bachman2019learning,tian2020contrastive, chen2020simple, chuang2020debiased}. It minimizes the following loss function:
\begin{align}
    &\gL_{\textnormal{InfoNCE}}(\textbf{s}) = -\log \frac{e^{s^+}}{e^{s^+} + \sum_{i=1}^K e^{s_i^-} } \nonumber \\
    &:=  -\log \frac{e^{f(x)^T g(v) / t}}{e^{f(x)^T g(v) / t} + \sum_{i=1}^K e^{f(x)^T g(v_i) / t}},
    \label{eq_infonce}
\end{align}
where $\textbf{s} = \{ s^+, \{s_i^-\}_{i=1}^K \}$, $s^+$ and $s_i^-$ are the scores of related (positive) and unrelated (negative) pairs and $t$ is the temperature parameter introduced to avoid gradient saturation. The expectation of the loss is taken over $(x,v) \sim P_{XV}$ and $K$ independent samples $v_i \sim P_V$, where $P_{XV}$ denotes the joint distribution over pairs of views such as transformations of the same image or co-occurring multimodal signals. Although InfoNCE has a functional form of the $(K+1)$-way softmax cross entropy loss, the model ultimately learns to classify whether a pair $(x, v)$ is positive or negative by maximizing/minimizing the positive score $s^+$/negative scores $s_i^-$. Therefore, InfoNCE under noisy views can be seen as binary classification with noisy labels. We acknowledge that similar interpretations have been made in prior works under different contexts~\citep{gutmann2010noise, wu2018unsupervised, tian2020contrastive}.

\vspace{-2mm}
\paragraph{Symmetric form of contrastive learning.} 
Now we turn to a functional form of contrastive learning that can achieve the symmetric property. Assume that we have a noise-tolerant loss function $\ell$ that satisfies the symmetry condition of \eqref{eq_sym}. We say a contrastive learning objective is symmetric if it accepts the following form
\begin{align}
    \label{eq_sym_form}
    \gL(\textbf{s}) = 
    \underbrace{\ell(s^+, 1)}_{\textnormal{Positive Pair}} + \lambda  \underbrace{\sum_{i=1}^K \ell(s_i^-, -1)}_{\textnormal{$K$ Negative Pairs}}
\end{align}
which consists of a collection of $(K+1)$ binary classification losses; $\lambda > 0$ is a density weighting term controlling the ratio between classes $1$ (positive pairs) and $-1$ (negative pairs). Reducing $\lambda$ places more weight on the positive score $s^+$, while setting $\lambda$ to zero recovers the negative-pair-free contrastive loss such as BYOL \cite{grill2020bootstrap}. 

Contrastive objectives that satisfy the symmetric form enjoy strong theoretical guarantees against noisy labels as described in \citet{ghosh2015making}, as long as we plug in the right contrastive loss function $\ell$ that satisfies the symmetry condition. Unfortunately, the InfoNCE loss~\cite{oord2018representation} does not satisfy the symmetry condition in the gradients w.r.t. $s^{+/-}$ (we provide the full derivations in Appendix~\ref{sec_infonce_asymmetric}). This motivates us to develop a new contrastive loss function that satisfies the symmetry condition, described next.

\section{Robust InfoNCE Loss}
\label{sec_rince}

\begin{figure}[t!]
\begin{lstlisting}[language=Python]
# pos: exponent for positive example
# neg: sum of exponents for negative examples
# q, lam: hyperparameters of RINCE

info_nce_loss = -log(pos/(pos+neg))
rince_loss = -pos**q/q + (lam*(pos+neg))**q/q
\end{lstlisting}
\vspace{-3mm}
\caption{\textbf{Pseudocode for RINCE.} The implementation only requires a small modification to the InfoNCE code.} 
\vspace{-3mm}
\label{fig_objective_code}
\end{figure}

\vspace{-1mm}
Based on the idea of robust symmetric classification loss, we present the following Robust InfoNCE (RINCE) loss:
\begin{align*}
    \gL^{\lambda, q}_{\textnormal{RINCE}}(\textbf{s}) = \frac{-e^{q \cdot s^+}}{q} + \frac{(\lambda \cdot (e^{s^+} + \sum_{i=1}^K e^{s_i^-}))^q}{q},
\end{align*}
where $q, \lambda \in (0, 1]$. Figure \ref{fig_objective_code} shows the pseudo-code of RINCE: it is simple to implement. When $q = 1$, RINCE becomes a contrastive loss that fully satisfies the symmetry property in the form of \eqref{eq_sym_form} with $\ell(s, y) = -ye^s$:
\vspace{-2mm}
\begin{align*}
     \gL_{\textnormal{RINCE}}^{\lambda, q=1}(\textbf{s}) = -(1 - \lambda)e^{s^+} + \lambda \sum_{i=1}^K e^{s_i^-}.
\end{align*}
Notice that the exponential loss $-ye^s$ satisfies the symmetric condition defined in \eqref{eq_sym} with $c=0$. Therefore, when $q \rightarrow 1$, we achieve robustness against noisy views in the same manner as binary classification with noisy labels.

In the limit of $q \rightarrow 0$, RINCE becomes asymptotically equivalent to InfoNCE, as the following lemma describes:
\begin{lemma}
For any $\lambda > 0$, it holds that
\begin{align*}
 &\lim_{q \rightarrow 0} \gL_{\textnormal{RINCE}}^{\lambda, q}(\textbf{s}) = \gL_{\textnormal{InfoNCE}}(\textbf{s}) + \log(\lambda);
 \\
  &\lim_{q \rightarrow 0} \frac{\partial}{\partial \textbf{s}}\gL_{\textnormal{RINCE}}^{\lambda, q}(\textbf{s}) = \frac{\partial}{\partial \textbf{s}}\gL_{\textnormal{InfoNCE}}(\textbf{s}).
\end{align*}
\label{lem_asym}
\vspace{-3mm}
\end{lemma}
We defer the proofs to Appendix~\ref{sec_proof}. Note that the convergence also holds for the derivatives: optimizing RINCE in the limit of $q \rightarrow 0$ is mathematically equivalent to optimizing InfoNCE. Therefore, by controlling $q \in (0, 1]$ we smoothly interpolate between the InfoNCE loss ($q \rightarrow 0$) and the RINCE loss in its fully symmetric form ($q \rightarrow 1$).

\subsection{Intuition behind RINCE}
\label{sec_tradeoff}
We now analyze the behavior of RINCE through the lens of exploration-exploitation trade-off. In particular, we reveal an implicit easy/hard positive mining scheme by inspecting the gradients of RINCE under different $q$ values, and show that we achieve stronger robustness (more exploitation) with larger $q$ at the cost of potentially useful clean hard positive samples (less exploration).

To simplify the analysis, we consider InfoNCE and RINCE with a single negative pair ($K=1$):
\begin{align*}
    &\gL_{\textnormal{InfoNCE}}(\textbf{s}) = -\log (e^{s^+} / (e^{s^+} + e^{s^-} ));
    \\
    &\gL_{\textnormal{RINCE}}^{\lambda, q}(\textbf{s}) = \frac{-e^{q \cdot s^+}}{q} + \frac{(\lambda \cdot (e^{s^+} + e^{s^-}))^q}{q}.
\end{align*}

\begin{figure}[!tb]
\begin{center}   \includegraphics[width=\linewidth]{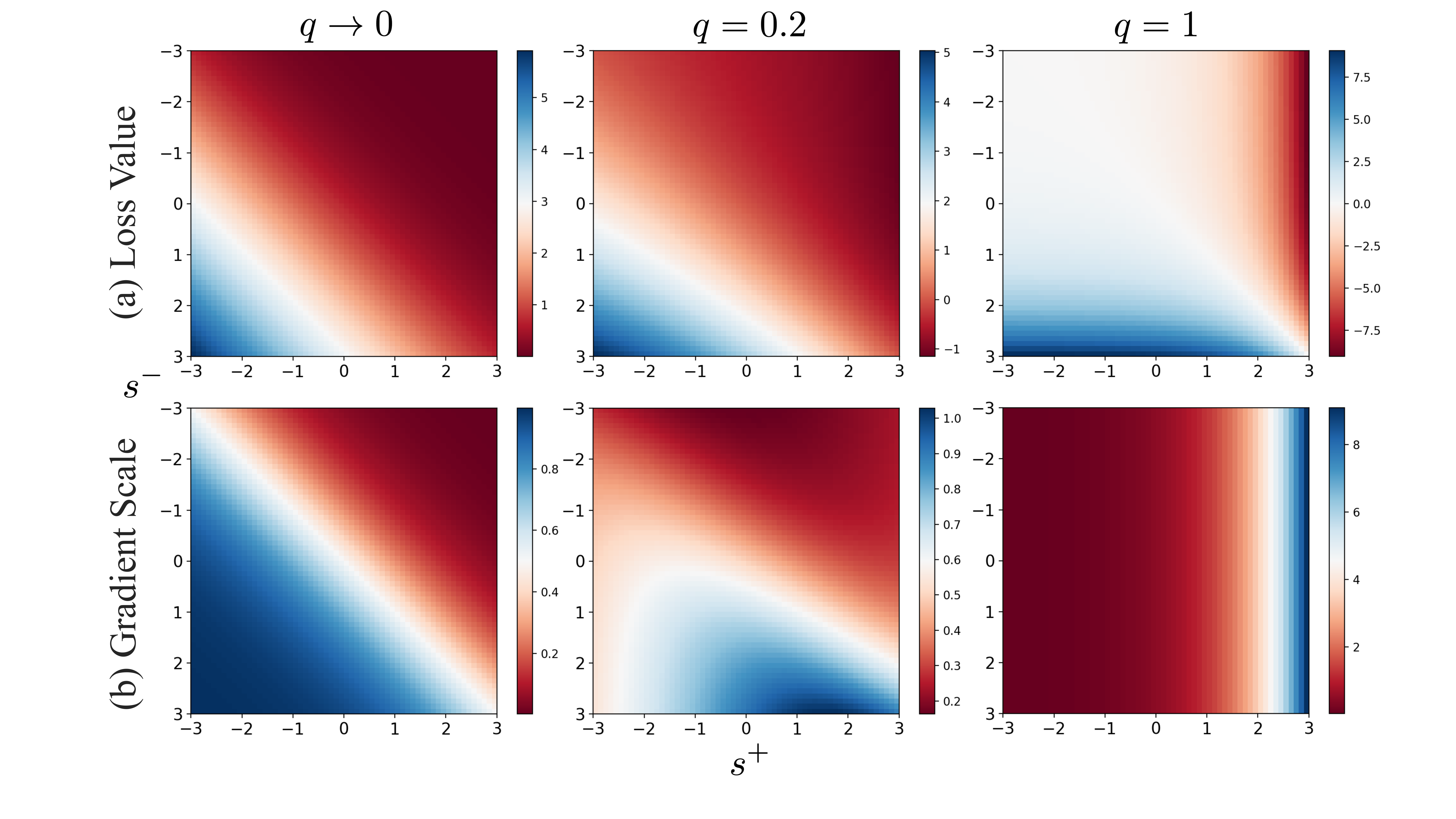}
\end{center}
\vspace{-4.5mm}
   \caption{\textbf{Loss Visualization.} We visualize the (a) loss value and the (b) gradient scale with respect to the positive score $s^+$ for different $q$ while setting $\lambda = 0.5$. The gradient scale of InfoNCE ($q \rightarrow 0$) is larger when the positive score is smaller (hard positive pair). In contrast, for fully symmetric RINCE ($q=1$), the gradient is larger when positive score is large (easy positive pair).} 
\label{fig_gradient}
\vspace{-3.5mm}
\end{figure}

We visualize the loss and the scale of the gradients with respect to positive scores $s^+$ in Figure \ref{fig_gradient}. Although the loss values are different for each $q$, they follow the same principle: The loss achieves its minimum when the positive score $s^+$ is maximized and the negative score $s^-$ is minimized. 

The interesting bit lies in the gradients. The InfoNCE loss ($q \rightarrow 0$) places more emphasis on \emph{hard} positive pairs, i.e., the pairs with \emph{low} positive scores $s^+$ (the left-most part in the plot). In contrast, the fully symmetric RINCE loss ($q = 1$) places more weights on \emph{easy} positive pairs (the right-most part). 
Note that both $q \rightarrow 0$ and $q \rightarrow 1$ naturally perform hard negative mining; both their derivatives put exponentially more weights on hard negative pairs.

This reveals an implicit trade-off between exploration (convergence) and exploitation (robustness). When $q \rightarrow 0$, the loss performs hard positive mining, providing faster convergence in the noise-free setting. But in the presence of noise, exploration is harmful; it wrongly puts higher weights to false positive pairs because noisy samples tend to induce larger losses \cite{arpit2017closer, han2018co, yu2019does, morgado2021robust}, and this could hinder convergence. In contrast, when $q \rightarrow 1$, we perform easy positive mining. This provides robustness especially against false positives; but this is done at the cost of exploration with clean hard positives. An important aspect here is that RINCE does not require an explicit form of noise estimator: the scores $s^+$ and $s^-$, and the relationship between the two (which is what the loss function measures) act as noise estimates. In practice, we set $q \in [0.1, 0.5]$ to strike the balance between exploration and exploitation.

\subsection{Theoretical Underpinnings}
\label{sec_theory}
Next, we provide an information-theoretic explanation on what makes RINCE robust against noisy views. In particular, we show that RINCE is a contrastive lower-bound of mutual information (MI) expressed in Wasserstein dependency measure (WDM) \citep{ozair2019wasserstein}, which provides superior robustness against sample noise compared to the Kullback–Leibler (KL) divergence thanks to strong geometric properties of the Wasserstein metric. We further show that, even in the presence of noise, RINCE is a lower bound of clean WDM, indicating its robustness against noisy views.

\vspace{-1em}\paragraph{Limitations of KL divergence in MI estimation.} Without loss of generality, let $f = g$ and consider $f = f' \circ \phi$, where $\phi$ is a representation encoder and $f'$ is a projection head~\citep{chen2020simple}. Also, let $P^\phi = \phi_\# P$ be the pushforward measure of $P$ with respect to $\phi$. It has been shown~\cite{poole2019variational,tian2020makes} that InfoNCE is a variational lower-bound of MI in the representation space expressed with KL-divergence:
\begin{align*}
  -\E \left[ \gL_{\textnormal{InfoNCE}}(\textbf{s}) \right] + \log(K) \leq I(\phi(X),\phi(V))& \\
  = D_{\textnormal{KL}}(P^\phi_{XV}, P^\phi_{X} P^\phi_{V})&.
\end{align*}
Intuitively, maximizing MI can be interpreted as maximizing the discrepancy between positive and negative pairs. However, prior works~\cite{mcallester2020formal,ozair2019wasserstein} have identified theoretical limitations of maximizing MI using the KL divergence: Because KL divergence is not a metric, it is sensitive to small differences in data samples regardless of the geometry of the underlying data distributions. Therefore, the encoder $\phi$ can capture limited information shared between $X$ and $V$ as long as the differences are sufficient to maximize the KL divergence. Note that this can be especially detrimental in the presence of noisy views, as the learner can quickly settle on spurious correlations in false positive pairs due to the absence of the actual shared information. 

\vspace{-1em}\paragraph{RINCE is a lower bound of WDM.} We now establish RINCE as a lower bound of WDM~\cite{ozair2019wasserstein}, which is proposed as a replacement for the KL divergence in MI estimation. 

WDM is based on the Wasserstein distance, a distance metric between probability distributions defined via an optimal transport cost. Letting $\mu$ and $\nu \in \textnormal{Prob}(\sR^d \times \sR^d)$ be two probability measures, we define the Wasserstein-$1$ distance with a Euclidean cost function as
\begin{align*}
    \gW(\mu, \nu) = \inf_{\pi \in \Pi(\mu, \nu)}  \E_{\substack{(X,V) \\ (X', V')} \sim \pi} \left[ \left\|X-X' \right\| + \left\|V-V' \right\| \right]
\end{align*}
where $\Pi(\mu, \nu)$ denotes the set of measure couplings whose marginals are $\mu$ and $\nu$, respectively. By virtue of symmetry when $q=1$, if $\lambda > 1 / (K+1)$, the Kantorovich-Rubinstein duality \citep{hormander2006grundlehren} implies that (full theorem in Appendix~\ref{sec_wmd_proof}):
\begin{align}
  -\E \left[\gL_{\textnormal{RINCE}}^{\lambda, q=1}(\textbf{s})\right] \leq L \cdot I_\gW(\phi(X),\phi(V))& \nonumber \\
  := L \cdot \gW(P_{XV}^\phi, P_{X}^\phi P_{V}^\phi)&,
\label{eq_iw}
\end{align}
where $I_\gW(\phi(X),\phi(V))$ is the WDM defined in \citep{ozair2019wasserstein} and $L$ is a constant that depends on $t, \lambda$, and the Lispchitz constant of the projection head $f$. Note that we are not aware of any work that showed it is possible to establish a similar bound with WDM for the InfoNCE loss.

This provides another explanation of what makes RINCE robust against noisy views. Unlike InfoNCE which maximizes the KL divergence, optimizing RINCE is equivalent to maximizing the WDM with a Lipschitz function. Equipped with a proper metric, this allows RINCE to measure the divergence between two distributions $P_{XV}^\phi$ and $P_{X}^\phi P_{V}^\phi$ without being overly sensitive to individual sample noise, as long as the noise does not alter the geometry of the  distributions. This also allows the encoder $\phi$ to learn more complete representations, as maximizing the Wasserstein distance requires the encoder to not only model the density ratio between the two distributions but also the optimal cost of transporting one distribution to another. 

\vspace{-1em}\paragraph{RINCE is still a lower bound of WDM even with noise.}
Finally, we show that RINCE still maximizes the noise-less WDM under additive noise, corroborating the robustness of RINCE. Let's consider a simple mixture noise model:
\begin{align*}
    P_{XV}^\eta = (1 - \eta) P_{XV} +  \eta P_{X}P_{V},
\end{align*}
where $\eta$ is the noise rate and the noisy joint distribution $P_{XV}^\eta$ is a weighted sum between the noise-less positive distribution $P_{XV}$ and negative distribution $P_X P_V$. Note that the marginals of $P_{XV}^\eta$ are still $P_X$ and $P_V$ by construction. The intuition behind mixture noise model is that when we draw positive pairs from $P_{XV}^\eta$, we obtain false positives from $P_X P_V$ with probability $\eta$. Via the symmetry of the contrastive loss, we can extend bound (\ref{eq_iw}) as follows (proof in Appendix \ref{sec_noisy_wdm}):
\begin{align*}
-\E_{P_{XV}^\eta} \left[\gL_{\textnormal{RINCE}}^{\lambda, q=1}(\textbf{s})\right] \leq (1-\eta) \cdot L \cdot I_\gW(\phi(X),\phi(V)).
\end{align*}
Comparing to the bound (\ref{eq_iw}), the right hand side is rewieghted with $(1-\eta)$. This implies that minimizing RINCE with noisy views still maximizes a lower bound of noise-less WDM. Despite the simplicity of the analysis, it intuitively relates dependency measures and the noisy views with interpretable bounds. It would be an interesting future direction to extend the analysis to more complicated noise models, e.g., $P_{XV}^\eta = (1 - \eta) P_{XV} +  \eta Q_{XV}$, where $Q$ is an unknown perturbation on positive distribution.

\section{Experiments}
\label{sec_exp}

We evaluate RINCE on various contrastive learning scenarios involving images (CIFAR-10 \citep{krizhevsky2009learning}, ImageNet \citep{deng2009imagenet}), videos (ACAV100M \citep{lee2021acav100m}, Kinetics400 \citep{kay2017kinetics}) and graphs (TUDataset \citep{morris2020tudataset}). Empirically, we find that RINCE is insensitive to the choice of $\lambda$; we simply set $\lambda = 0.01$ for all vision experiments and $\lambda = 0.025$ for graph experiments. 

\begin{figure}[tp]
\begin{center}   \includegraphics[width=\linewidth]{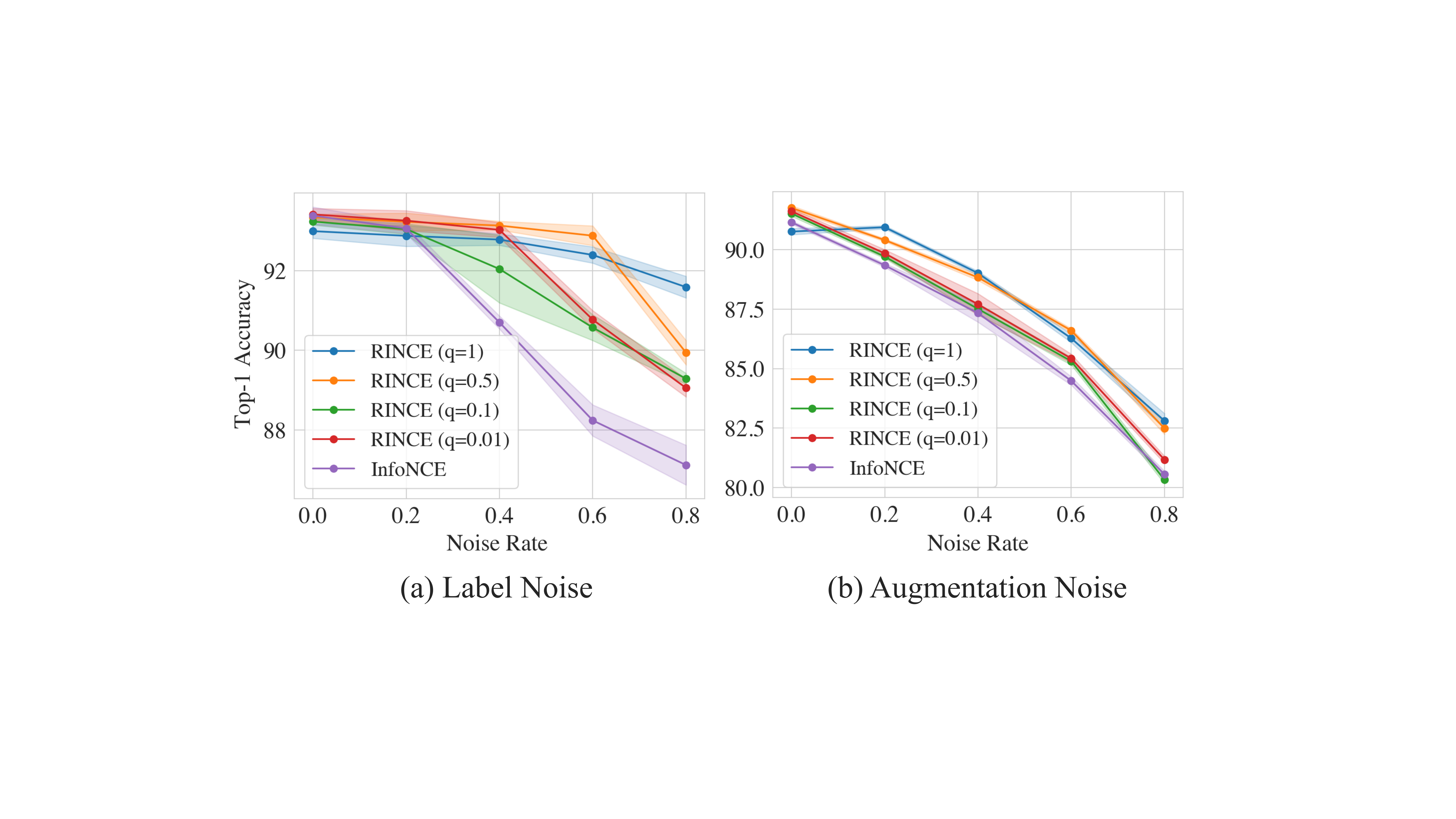}
\end{center}
\vspace{-4mm}
   \caption{\textbf{Noisy CIFAR-10.} We show the top1 accuracy of RINCE with different values of $q$ across different noise rate $\eta$. Large $q$ ($q = 0.5, 1$) leads to better robustness, while smaller $q$ ($q = 0.01$) performs similar to InfoNCE ($q \rightarrow 0$). }
\label{fig_cifar}
\vspace{-2.0mm}
\end{figure}
\begin{figure}[tp]
\begin{center}   \includegraphics[width=0.9\linewidth]{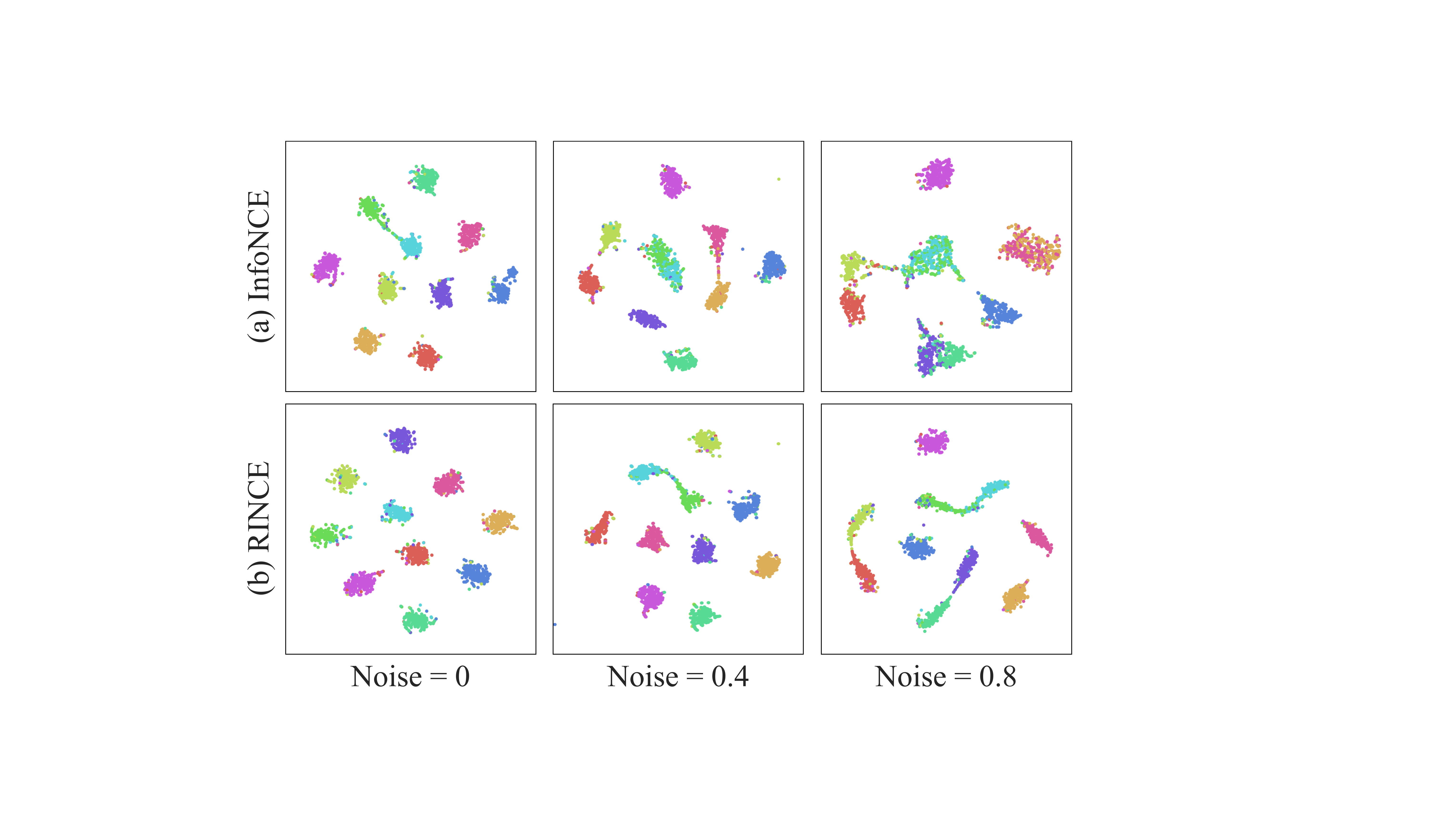}
\end{center}
\vspace{-3.5mm}
   \caption{\textbf{t-SNE Visualization on CIFAR-10 with label noise.} Colors indicate classes. RINCE leads to better class-wise separation than the InfoNCE loss in both noise-less and noisy cases.} 
\label{fig_tsne}
\vspace{-3.0mm}
\end{figure}

\subsection{Noisy CIFAR-10}

We begin with controlled experiments on CIFAR-10 to verify the robustness of RINCE against synthetic noise by controlling the noise rate $\eta$. We consider two noise types:

\textbf{Label noise:} We start with the case of supervised contrastive learning~\cite{khosla2020supervised} where positive pairs are different images of the same label. This allows us to control noise in the traditional sense, i.e., learning with noisy labels. Similar to \citep{zhang2018generalized}, we flip the true labels to semantically related ones, e.g., CAT $\leftrightarrow$ DOG with probability $\eta/2$. This is commonly referred to as class-dependent noise~\citep{han2018co, zhang2018generalized, wang2019symmetric}. 

\textbf{Augmentation noise:} We consider the self-supervised learning scenario and vary the crop size during data augmentation similar to \cite{tian2020makes}, i.e., after applying all the transformations as in SimCLR~\citep{chen2020simple}, images are further cropped into $1/5$ of their original size with probability $\eta$. This effectively controls the noise rate as cropped patches will most likely to be too small to contain any shared information.

Figure \ref{fig_cifar} shows the results of SimCLR trained with InfoNCE and RINCE with different choices of $q$ and $\lambda$. When the augmentation noise is present, e.g., $\eta = 0.4$, the accuracy of InfoNCE drops from $91.14\%$ to $87.33\%$. In contrast, the robustness of RINCE is enhanced by increasing $q$, achieving $89.01\%$ when $q = 1.0$,. InfoNCE also fails to address label noise and suffers from significant performance drop ($93.38\% \rightarrow 87.11\%$ when $\eta=0.8$). In comparison, RINCE retains the performance even when the noise rate is large ($91.59\%$ for $q=1.0$). In both cases, reducing the value of $q$ makes the performance of RINCE closer to InfoNCE, verifying our analysis in Lemma \ref{lem_asym}.

Figure \ref{fig_tsne} shows t-SNE visualization~\citep{van2008visualizing} of representations learned with InfoNCE and RINCE ($q = 1.0$) under different label noise. As the noise rate increases, representations of different classes start to tangle up for InfoNCE, while RINCE still achieves decent class-wise separation.

\begin{table}[tp]
    \small
      \centering
        \begin{tabularx}{0.45\textwidth}{r|c|cc}
            \toprule
            Method & $\Delta$ to SimCLR~\cite{chen2020simple} & Top 1 & Top 5  \\
            \midrule
            \midrule
            Supervised~\cite{he2016deep} & N/A & 76.5 & -
            \\
            \midrule
            SimSiam~\cite{chen2021exploring} & No negative pairs & 71.3 & - 
            \\
            BYOL~\cite{grill2020bootstrap} & No negative pairs & 74.3 & 91.6  
            \\
            Barlow Twins~\cite{zbontar2021barlow} & Redundancy reduction & 73.2 & 91.0
            \\
            SwAV~\cite{caron2020unsupervised} & Cluster discrimination & 75.3 & -
            \\
            \midrule
            \midrule    
            SimCLR~\cite{chen2020simple} & None & 69.3 & 89.0
            \\
            +RINCE (Ours) & Symmetry controller $q$ & \textbf{70.0} & \textbf{89.8}
            \\
            \midrule
            MoCo~\cite{he2020momentum} & Momentum encoder & 60.6 & -  
            \\
            MoCov2~\cite{chen2020improved} & Momentum encoder & 71.1 & 90.1 
            \\
            MoCov3~\cite{chen2021empirical} & Momentum encoder & 73.8 & -
            \\
            +RINCE (Ours) & Symmetry controller $q$ & \textbf{74.2} & \textbf{91.8} 
            \\
            \bottomrule
            \multicolumn{4}{l}{
            \footnotesize \it 
            }
        \end{tabularx}
    \vspace{-.5em}
    \caption{\textbf{Linear Evaluation on ImageNet.} All the methods use ResNet-50 \cite{he2016deep}
as backbone architecture with 24M parameters. Note that RINCE subsumes InfoNCE when $q \rightarrow 0$.
}
\label{table_imgnet}
 \vspace{-3.5mm}
\end{table}

\subsection{Image Contrastive Learning}
We verify our approach on the well-established ImageNet benchmark~\cite{deng2009imagenet}. We adopt the same training protocol and hyperparameter settings of SimCLR~\citep{chen2020simple} and MoCov3~\citep{chen2021empirical} and simply replace the InfoNCE with our RINCE loss ($q = 0.1$ and $q=0.6$, respectively) as shown in Figure~\ref{fig_objective_code}. Table \ref{table_imgnet} shows that RINCE improves InfoNCE (SimCLR and MoCov3) by a non-trivial margin. We also include results from the SOTA baselines, where they improve SimCLR by introducing dynamic dictionary plus momentum encoder (MoCo-v1/v2/v3~\cite{he2020momentum, chen2020improved, chen2021empirical}), removing negative pairs plus the stop-gradient trick (SimSiam \cite{chen2021exploring}, BYOL \cite{grill2020bootstrap}), or online cluster assignment (SwAV \cite{caron2020unsupervised}). In comparison, our work is orthogonal to the recent developments, and the existing tricks can be applied along with RINCE.

Figure \ref{fig_imgnet} shows the positive pairs from SimCLR augmentations and the corresponding positive scores $s^+ = f(x)^T g(v)$ output by trained RINCE model. Examples with lower positive scores contain pairs that is less informative to each other, while semantically meaningful pairs often have higher scores. This implies that positive scores are good noise detectors, and down-weighting the samples with lower positive score brings robustness during training, verifying our analysis in section \ref{sec_tradeoff}.

\subsection{Video Contrastive Learning}
We examine our approach in the audio-visual learning scenario using two video datasets: Kinetics400~\citep{kay2017kinetics} and ACAV100M~\citep{lee2021acav100m}. Here, we find that simple \emph{$q$-warmup} improves the stability of RINCE, i.e., $q$ starts at $0.01$ and linearly increases to $0.4$ until the last epoch. We apply this to all RINCE models in this section. As we show below, RINCE outperforms SOTA noise-robust contrastive methods~\cite{morgado2021audio,morgado2021robust} on Kinetics400, while also providing scalability and computational efficiency compared to InfoNCE. 

\newcommand{\ActionSOTAbyDB}{
\begin{tabular}{r|cccc}
\toprule
\bf Method & \bf \begin{tabular}{c}\bf Backbone\end{tabular} & \bf \begin{tabular}{c}Finetune \\Input Size\end{tabular} & \bf HMDB & \bf UCF  \\
\hline\hline
 3D-RotNet \citep{jing2018self} & R3D-18 & $16\!\times\!112^2$ & 33.7 & 62.9  \\
 ClipOrder \citep{xu2019self}   & R3D-18 & $16\!\times\!112^2$ & 30.9 & 72.4 \\
 DPC \citep{han2019video}              & R3D-18 & $25\!\times\!128^2$  & 35.7 & 75.7 \\
 CBT \citep{sun2019learning}                & S3D & $16\!\times\!112^2$ & 44.6 & 79.5  \\
 AVTS \citep{korbar2018cooperative}        & MC3-18 & $25\!\times\!224^2$ & 56.9 & 85.8 \\
 SeLaVi \citep{asano2020labelling}               & R(2+1)D-18 & $32\!\times\!112^2$  & 47.1 & 83.1\\
 XDC \citep{alwassel2019self} & R(2+1)D-18 & $32\!\times\!224^2$  & 52.6 & 86.8 \\
 Robust-xID \citep{morgado2021robust}   & R(2+1)D-18 & $32\!\times\!224^2$ & 55.0 & 85.6  \\
 Cross-AVID \citep{morgado2021audio}   & R(2+1)D-18 & $32\!\times\!224^2$ & 59.9 & 86.9 \\
 AVID+CMA \citep{morgado2021audio}      & R(2+1)D-18 & $32\!\times\!224^2$  & 60.8 & 87.5\\
\midrule
InfoNCE (Ours)   & R(2+1)D-18 & $32\!\times\!224^2$ & 57.8  & 88.6 \\
 RINCE (Ours)     & R(2+1)D-18 & $32\!\times\!224^2$  & \textbf{61.6} & \textbf{88.8} \\
\midrule
\gray{GDT} \citep{patrick2021compositions}        & \gray{R(2+1)D-18} & \gray{$30\!\times\!112^2$} & \gray{62.3$^\ast$} & \gray{90.9$^\ast$}  \\
\bottomrule 
\multicolumn{5}{l}{\footnotesize \it ${}^*$Based on an advanced hierarchical data augmentation during pretraining.}
\end{tabular}
}
\begin{table}[t!]
    \centering
    \resizebox{\columnwidth}{!}{\ActionSOTAbyDB}
    \caption{\textbf{Kinetics400-pretrained performance on UCF101 and HMDB51} (top-1 accuracy). Ours use the same data augmentation approach as Cross-AVID and AVID+CMA, while GDT uses an advanced hierarchical sampling process.}
\label{table_kinetic}
\vspace{-3.5mm}
\end{table}

\vspace{-1em}\paragraph{Kinetics400}

For fair comparison to SOTA, we follow the same experimental protocol and hyperparameter settings of \cite{morgado2021audio} and simply replace their loss functions with InfoNCE and RINCE loss as shown in Figure~\ref{fig_objective_code}. We use the same network architecture, i.e., 18-layer R(2+1)D video encoder~\citep{tran2018closer}, 9-layer VGG-like audio encoder, and 3-layer MLP projection head producing 128-dim embeddings. We use the ADAM optimizer \citep{kingma2014adam} for 400 epochs with 4,096 batch size, 1e-4 learning rate and 1e-5 weight decay. The pretrained encoders are finetuned on UCF-101 \citep{soomro2012ucf101} and HMDB-51 \citep{kuehne2011hmdb} with clips composed of 32 frames of size 224 $\times$ 224. We defer the full experimental details to Appendix \ref{sec_exp_detail}. 

Table \ref{table_kinetic} shows that RINCE outperforms most of the baseline approaches, including Robust-xID~\cite{morgado2021robust} and AVID+CMA~\cite{morgado2021audio} which are recent InfoNCE-based SOTA methods proposed to address the noisy view issues in audio-visual contrastive learning. Considering the only change required is the simple replacement of the InfoNCE with our RINCE loss, the results clearly show the effectiveness of our approach. The simplicity means we can easily apply RINCE to a variety of InfoNCE-based approaches, such as GDT~\cite{patrick2021compositions} that uses advanced data augmentation mechanisms to achieve SOTA results.

\vspace{-1em}\paragraph{ACAV100M}
We conduct an in-depth analysis of RINCE on ACAV100M~\citep{lee2021acav100m}, a recent large-scale video dataset for self-supervised learning. Compared to Kinetics400 which is limited to human actions, ACAV100M contains videos ``in-the-wild'' exhibiting a wide variety of audio-visual patterns. The unconstrained nature of the dataset makes it a good benchmark to investigate the robustness of RINCE to various types of real-world noise, e.g., background music, overdubbed audio, studio narrations, etc.

\begin{figure}[!tb]
\begin{center}   \includegraphics[width=0.9\linewidth]{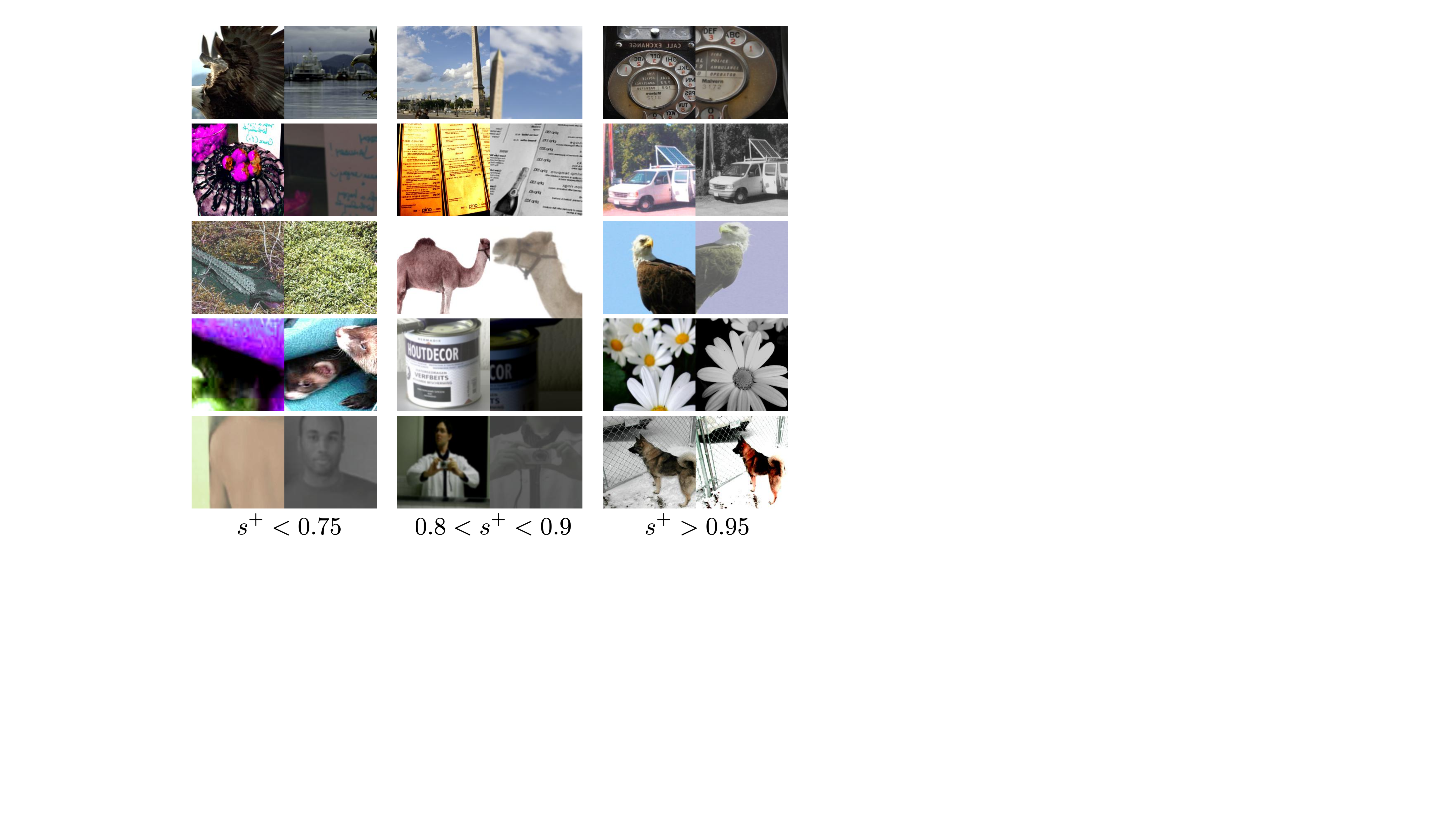}
\end{center}
\vspace{-3.5mm}
   \caption{\textbf{Positive pairs and their scores.} The positive scores $s^+ \in [-1, 1]$ are output by the trained RINCE model (temperature $=$ 1). Pairs that have lower scores are visually noisy, while informative pairs often have higher scores.}
\label{fig_imgnet}
\vspace{-4.0mm}
\end{figure}

We focus on evaluating the (a) scalability and (b) convergence rate of RINCE, thereby answering the question: \emph{Will it retrain its edge over InfoNCE (a) even in the large-scale regime and (b) with a longer training time?} We follow the same experimental setup as described above, but reduce the batch size to 512 and report the results only on the first split of UCF-101 to make our experiments tractable.

Figure \ref{fig_acav} (a) shows the top-1 accuracy of RINCE and InfoNCE across different data scales and training epochs. RINCE outperforms InfoNCE by a large margin at every data scale. In terms of the convergence rate, RINCE is comparable to or even outperforms fully-trained (200 epochs) InfoNCE models with only 100 or fewer epochs. Figure \ref{fig_acav} (b) gives a closer look at the convergence at 50K and 200K scales. Interestingly, InfoNCE saturates and even degenerates after epoch 150, while RINCE keeps improving. This verifies our analysis in section \ref{sec_tradeoff}: InfoNCE can overfit noisy samples due to its exploration property, while RINCE downweights them and continue to obtain the learning signal from clean ones, achieving robustness against noise.

\begin{figure}[tp]
\begin{center}   \includegraphics[width=\linewidth]{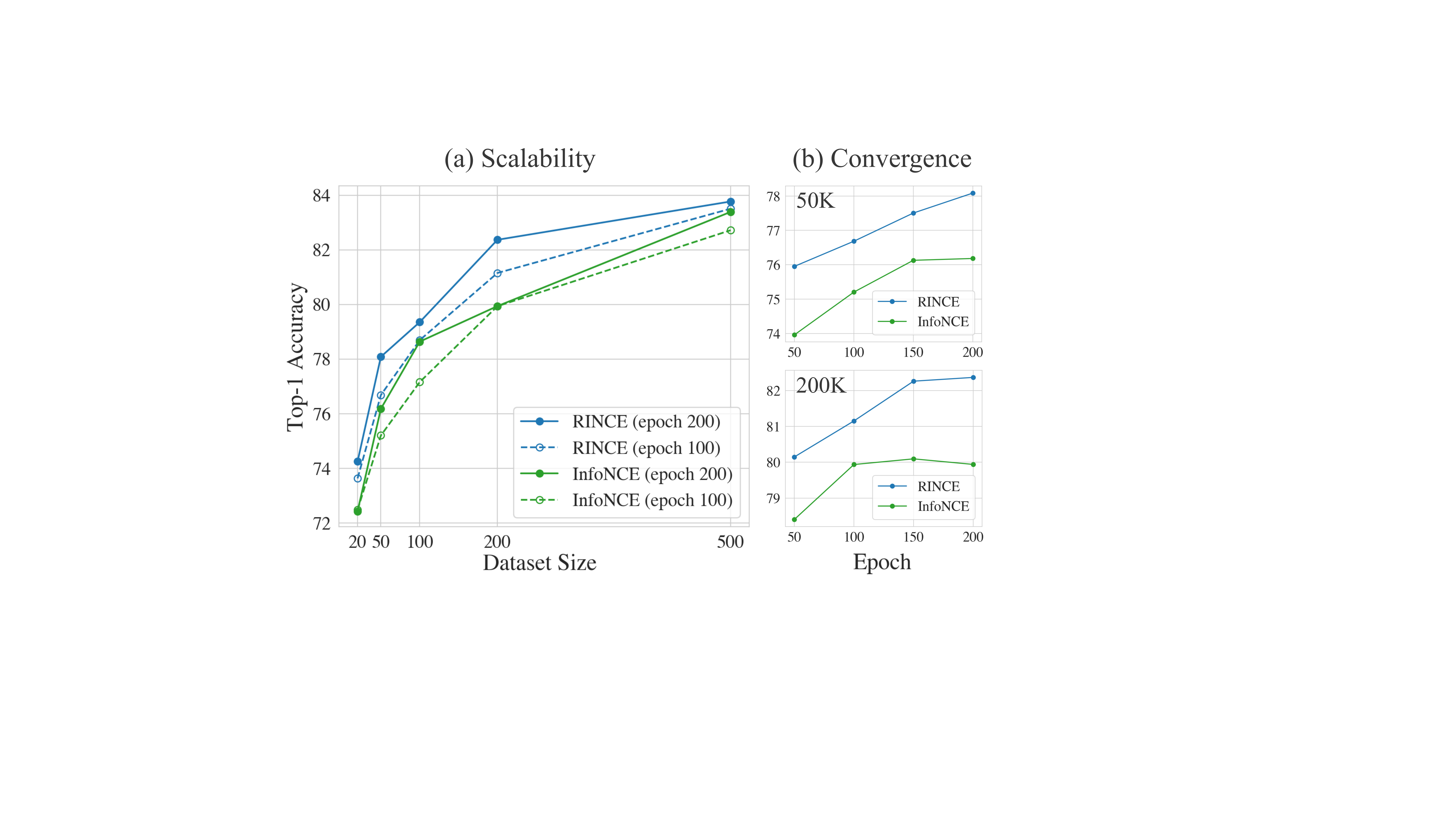}
\end{center}
\vspace{-4.5mm}
   \caption{
   \textbf{RINCE outperforms InfoNCE with fewer epochs across different scales} The results are based on ACAV100M-pretrained models transferred to UCF-101.
   }
\label{fig_acav}
\vspace{-3mm}
\end{figure}

\subsection{Graph Contrastive Learning}
To see whether the modality-agnostic nature of RINCE applies beyond image and video data, we examine our approach on TUDataset \citep{morris2020tudataset}, a popular benchmark suite for graph inference on molecules (BZR, NCI1), bioinformatics (PROTEINS), and social network (RDT-B, IMDB-B). Unlike vision datasets, data augmentation for graphs requires careful engineering with domain knowledge, limiting the applicability of InfoNCE-type contrastive objectives. 

For fair comparison, we follow the protocol of \citep{you2020graph} and train graph isomorphism networks~\citep{xu2018powerful} with four types of data augmentation: node dropout, edge perturbation, attribute masking, and subgraph sampling. We train models using ADAM~\citep{kingma2014adam} for 20 epochs with a learning rate 0.01 and report mean and standard deviation over 5 independent trials. We set $q = 0.1$ for all the experiments in this section. 

Table~\ref{table_graph} shows that RINCE outperforms three SOTA InfoNCE-based contrastive methods, GraphCL and JOAO /JOAOv2, setting the new records on all four datasets. GraphCL applies different augmentations for different datasets, while JOAO/JOAOv2 require solving bi-level optimization to choose optimal augmentation per dataset. In contrast, we apply the same augmentation across all four datasets and achieve competitive performance, demonstrating its generality and robustness. In Figure~\ref{fig_graph}, we control perturbation rate by applying three augmentation types (node dropout, edge perturbation, attribute masking) to different \% of nodes/edges. We show results on two datasets most sensitive to augmentation. Again, RINCE consistently outperform InfoNCE and has relatively smaller variances when the noise rate increases.

\begin{table}[tp]
     \centering
     \footnotesize
    \begin{tabularx}{0.47\textwidth}{r|*{5}{c}}
    \toprule
    \bf Methods & \bf RDT-B & \bf NCI1 & \bf PROTEINS & \bf DD &   \\
    \hline
    \hline
    node2vec \citep{grover2016node2vec} & -& 54.9$\pm$1.6 & 57.5$\pm$3.6 & -   \\
    sub2vec \citep{adhikari2018sub2vec} &71.5$\pm$0.4& 52.8$\pm$1.5 & 53.0$\pm$5.6&  -   \\
    graph2vec \citep{narayanan2017graph2vec} & 75.8$\pm$1.0& 73.2$\pm$1.8 & 73.3$\pm$2.1 & -  \\
    InfoGraph \citep{velivckovic2018deep} & 82.5$\pm$1.4& 76.2$\pm$1.1 & 74.4$\pm$0.3 & 72.9$\pm$1.8     \\
    GraphCL \citep{you2020graph} &89.5$\pm$0.8& 77.9$\pm$0.4 & 74.4$\pm$0.5  & 78.6$\pm$0.4  \\
    JOAO \citep{you2021graph} &85.3$\pm$1.4 & 78.1$\pm$0.5 & 74.6$\pm$0.4 & 77.3$\pm$0.5 \\
    JOAOv2 \citep{you2021graph} &86.4$\pm$1.5& 78.4$\pm$0.5 & 74.1$\pm$1.1  & 77.4$\pm$1.2  \\
    \midrule
    InfoNCE$^\ast$ (Ours)& 89.9$\pm$0.4 & 78.2$\pm$0.8 & 74.4$\pm$0.5 & 78.6$\pm$0.8 
    \\
    RINCE (Ours) & \bf 90.9$\pm$0.6 & \bf 78.6$\pm$0.4 & \bf 74.7$\pm$0.8 & \bf 78.7$\pm$0.4  \\
    \bottomrule
    \multicolumn{6}{l}{
            \footnotesize \it ${}^*$GraphCL \citep{you2020graph} but uses the same data augmentation as RINCE.
    }
        \end{tabularx}
    \vspace{-1mm}
    \caption{\textbf{Self-supervised representation learning on TUDataset}: The baseline results are excerpted from the published papers.
}
\label{table_graph}
\end{table}

\begin{figure}[tp]
\begin{center}  
\vspace{-1mm}
\includegraphics[width=\linewidth]{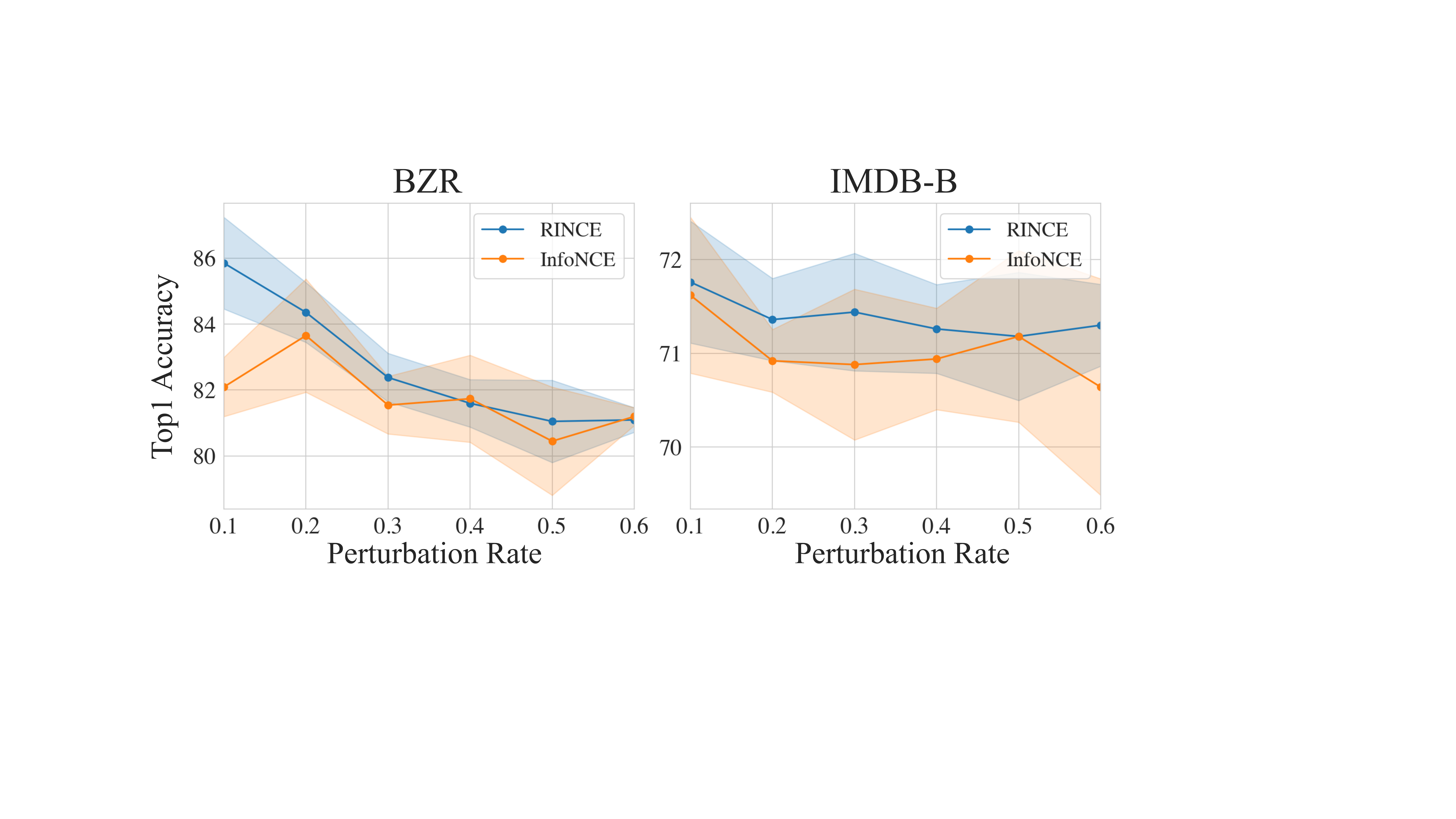}
\end{center}
\vspace{-3mm}
   \caption{\textbf{Performance v.s. Perturbation  Rate:} We increase the perturbation rate of node dropping, edge perturbation, and attribute masking from 10\% to 60\%. RINCE outperforms InfoNCE in terms of accuracy and variance when perturbation enhances.}
\label{fig_graph}
\vspace{-4mm}
\end{figure}

\vspace{-2mm}
\section{Conclusion}
\vspace{-2mm}

We presented \emph{Robust InfoNCE (RINCE)} as a simple drop-in replacement for the InfoNCE loss in contrastive learning. Despite its simplicity, it comes with strong theoretical justifications and guarantees against noisy views. Empirically, we provided extensive results across image, video, and graph contrastive learning scenarios demonstrating its robustness against a variety of realistic noise patterns. 

\vspace{-3mm}
\paragraph{Acknowledgements}
This work was in part supported by NSF Convergence Award 6944221 and ONR MURI 6942251.

{\small
\bibliographystyle{IEEEtranN}
\bibliography{egbib}
}

\clearpage
\onecolumn 
\appendix

\section{Theory and Proofs}
\label{sec_proof}
\subsection{Proof of Lemma \ref{lem_asym}: From RINCE to InfoNCE}
We show that RINCE becomes asymptotically equivalent to InfoNCE when $q \rightarrow 0$. In particular, we prove the convergence of RINCE and its derivative in the limit of $q \rightarrow 0$. 

\begin{proof}
We first prove the convergence in the function space with the L’Hôpital’s rule:
\begin{align*}
    &\lim_{q \rightarrow 0} \gL^{\lambda, q}_{\textnormal{RINCE}}(\textbf{s}) 
    \\
    &= \lim_{q \rightarrow 0} \frac{-e^{q \cdot s^+}}{q} + \frac{\left(\lambda \cdot (e^{s^+} + \sum_{i=1}^K e^{s_i^-})\right)^q}{q}
    \\
    &= \lim_{q \rightarrow 0} \frac{1 -e^{q \cdot s^+}}{q} + \frac{-1 + \left(\lambda \cdot (e^{s^+} + \sum_{i=1}^K e^{s_i^-})\right)^q}{q}
    \\
    &= \lim_{q \rightarrow 0} \frac{1 -e^{q \cdot s^+}}{q} + \lim_{q \rightarrow 0} \frac{-1 + \left(\lambda \cdot (e^{s^+} + \sum_{i=1}^K e^{s_i^-}) \right)^q}{q}
    \\
    &=-\log(e^{s^+}) + \log\left(\lambda(e^{s^+} + \sum_{i=1}^K e^{s_i^-}) \right) \tag{L’Hôpital’s rule}
    \\
    &= -\log \frac{e^{s^+}}{\lambda \left(e^{s^+} + \sum_{i=1}^K e^{s_i^-} \right)} 
    \\
    &= \gL_{\textnormal{InfoNCE}}(\textbf{s})  + \log(\lambda).
\end{align*}

To prove the convergence in its derivative, we analyze the derivative with respect to the positive score $s^+$ and the negative score $s_i^-$. We begin with RINCE:
\begin{align*}
    \textnormal{(positive score)} \quad\quad &\lim_{q \rightarrow 0} \frac{\partial}{\partial s^+} \gL^{\lambda, q}_{\textnormal{RINCE}}(\textbf{s}) 
    \\
    =& \lim_{q \rightarrow 0} \frac{\partial}{\partial s^+} \frac{-e^{q \cdot s^+}}{q} +  \frac{\partial}{\partial s^+} \frac{\left(\lambda \cdot (e^{s^+} + \sum_{i=1}^K e^{s_i^-})\right)^q}{q}
    \\
    =& \lim_{q \rightarrow 0} -e^{q \cdot s^+} + (\lambda \cdot (e^{s^+} + \sum_{i=1}^K e^{s_i^-}))^{q-1} \cdot \lambda \cdot e^{s^+}
    \\
    =& -1 + \frac{e^{s^+}}{e^{s^+} + \sum_{i=1}^K e^{s_i^-}};
    \\
    \textnormal{(negative score)} \quad\quad &\lim_{q \rightarrow 0} \frac{\partial}{\partial s_i^-} \gL^{\lambda, q}_{\textnormal{RINCE}}(\textbf{s}) 
    \\
    =& \lim_{q \rightarrow 0} \frac{\partial}{\partial s_i^-} \frac{\left(\lambda \cdot (e^{s^+} + \sum_{i=1}^K e^{s_i^-})\right)^q}{q}
    \\
    =& \lim_{q \rightarrow 0} (\lambda \cdot (e^{s^+} + \sum_{i=1}^K e^{s_i^-}))^{q-1} \cdot \lambda \cdot e^{s_i^-}
    \\
    =& \frac{e^{s_i^-}}{e^{s^+} + \sum_{i=1}^K e^{s_i^-}}.
\end{align*}
We can see that the derivatives match the ones of InfoNCE
\begin{align*}
    \textnormal{(positive score)} \quad\quad &\frac{\partial}{\partial s^+} \gL_{\textnormal{InfoNCE}}(\textbf{s}) 
    \\
    &=\frac{\partial}{\partial s^+} - \log \frac{e^{s^+}}{e^{s^+} + \sum_{i=1}^K e^{s_i^-}}
    \\
    &= - \frac{e^{s^+}}{e^{s^+} + \sum_{i=1}^K e^{s_i^-}} \cdot \frac{(e^{s^+} + \sum_{i=1}^K e^{s_i^-}) \cdot e^{s^+} - e^{2s^+}}{(e^{s^+} + \sum_{i=1}^K e^{s_i^-})^2}
    \\
    &= -1 + \frac{e^{s^+}}{e^{s^+} + \sum_{i=1}^K e^{s_i^-}};
    \\
    \textnormal{(negative score)} \quad\quad &\frac{\partial}{\partial s_i^-} \gL_{\textnormal{InfoNCE}}(\textbf{s}) 
    \\
    &=-\frac{e^{s^+} + \sum_{i=1}^K e^{s_i^-}}{e^{s_i^-}} \cdot \frac{-e^{s^+} \cdot e^{s_i^-}}{(e^{s^+} + \sum_{i=1}^K e^{s_i^-})^2}
    \\
    &= \frac{e^{s_i^-}}{e^{s^+} + \sum_{i=1}^K e^{s_i^-}}.
\end{align*}

\end{proof}

\subsection{Noisy Risk Bound for Exponential Loss}
\label{sec_exp_bound}

We justify the robustness of RINCE when $q=1$ by extending \citet{ghosh2015making}'s theorem to the exponential loss. The proof technique can be applied to other bounded symmetric classification losses.

\begin{corollary}
Consider the setting of \citet{ghosh2015making} and the exponential loss function $\gL(s, y) = -ye^s$. Let $f_\eta^\ast = \arginf_{f \in \gF} R_{\gL}^\eta(f)$ be the minimizer of the noisy risk and $\epsilon = \inf_{f \in \gF} R_{\gL}(f)$ be the optimal risk. If $\eta_x \leq \eta_{\max} < 0.5 $ for all $x \in \gX$. If the prediction score is bounded by $s_{\max}$, we have $R(f_\eta^\ast) \leq (\epsilon + 2  \eta_{\max} e^{s_{\max}}) / (1 - 2 \eta_{\max})$.
\label{cor_noisy}
\end{corollary}
\begin{proof}
Consider a binary classification loss with the following form:
\begin{align*}
    \tilde{\gL}_{x}(f(x), y) = B + \gL_{x}(f(x), y) =B - y \cdot e^{f(x)} \geq 0,
\end{align*}
where the prediction score $f(x)$ is bounded by $s_{\max} = \log(B)$. Note that the boundedness assumption holds for general representation learning on hypersphere, where the prediction score is the inner product between normalized feature vectors. Importantly, the loss satisfies 
\begin{align*}
    \tilde{\gL}(f(x), 1) +   \tilde{\gL}(f(x), -1) = 2B.
\end{align*}
By construction, the optimal risk takes the following value:
\begin{align*}
    \inf_{f \in \gF} R_{\tilde{\gL}}(f) = \inf_{f \in \gF} \E_{x \sim \mu} [\tilde{\gL}(f(x), y_x)] = \epsilon + B := \tilde{\epsilon},
\end{align*}
and $f^\ast = \arginf_{f \in \gF} R_{\tilde{\gL}}(f)$. Note that $f^\ast$ is also a minimizer w.r.t. the original loss $\gL$( $f^\ast = \arginf_{f \in \gF} R_{\gL}(f)$), as an additive constant will not change the optimum solutions. Expanding the noisy risk gives
\begin{align*}
    R_{\tilde{\gL}}^\eta(f) &= \E_{(x,y) \sim \mu} [(1 - \eta_{x})\tilde{\gL}(f(x), y_x) + \eta_{x} \tilde{\gL}(f(x), -y_x)]
    \\
    &= \E_{x \sim \mu} [(1 - \eta_{x})\tilde{\gL}(f(x), y_x) + \eta_{x} (2B - \tilde{\gL}(f(x), y_x))] \tag{Symmetry}
    \\
    &= \E_{x \sim \mu} [(1 - 2\eta_{x})\tilde{\gL}(f(x), y_x) ] + 2B \E_{x \sim \mu}[\eta_{x}] .
\end{align*}
Let $f_\eta^\ast = \arginf R_\gL^\eta(f_\eta^\ast) = \arginf R_{\tilde{\gL}}^\eta(f_\eta^\ast)$, we have
\begin{align*}
&R_{\tilde{\gL}}^\eta(f^\ast) -  R_{\tilde{\gL}}^\eta(f_\eta^\ast) = \E_{x \sim \mu} [(1 - 2\eta_{x})(\tilde{\gL}(f^\ast(x), y_x) - \tilde{\gL}(f_\eta^\ast(x), y_x)) ] \geq 0
\end{align*}
since $f_\eta^\ast$ is the minimizer of $R_{\tilde{\gL}}^\eta$, which implies that
\begin{align*}
    &E_{x \sim \mu} [(1 - 2\eta_{ x})\tilde{\gL}(f_\eta^\ast(x), y_x)  ] 
\leq E_{x \sim \mu} [(1 - 2\eta_{ x})\tilde{\gL}(f^\ast(x), y_x) ] \leq \tilde{\epsilon},
\end{align*}
since $0 < 1 - 2\eta_{x} \leq 1$ by assumption. Let $\eta_{max} = \sup_{x \in \gX} \eta_{x}$, we have
\begin{align*}
    (1 - 2 \eta_{\max}) E_{x \sim \mu} [\tilde{\gL}(f_\eta^\ast(x), y_x)  ] \leq \epsilon,
\end{align*}
since the loss is non-negative, which implies
\begin{align*}
    R_{\tilde{\gL}}(f_\eta^\ast) \leq \frac{\tilde{\epsilon}}{1 - 2 \eta_{\max}}.
\end{align*}

Finally, we recover the original exponential loss without the additive term $B$. Plugging the form we have
\begin{align*}
    B + R_{\gL}(f_\eta^\ast) \leq \frac{\epsilon + B}{1 - 2 \eta_{max}},
\end{align*}
which implies
\begin{align*}
    R_{\gL}(f_\eta^\ast) \leq \frac{\epsilon + B}{1 - 2 \eta_{max}} - B = \frac{\epsilon + 2B\eta_{max}}{1 - 2 \eta_{max}} .
\end{align*}
For exponential loss, setting $B$ to $e^{s_{\textnormal{max}}}$ completes the proof.
\end{proof}
For instance, when the noise level is $40 \%$, we have $R_{\gL}(f_\eta^\ast) \leq 5 \epsilon + 4B$. Note that the prediction score is bounded by $1/t$ in our case as the representations are projected onto the unit hypersphere.

\subsection{Lower bound of Wasserstein Distance}
\label{sec_wmd_proof}

We now establish RINCE as a lower bound of WDM~\cite{ozair2019wasserstein}. WDM is based on the Wasserstein distance, a distance metric between probability distributions defined via an optimal transport cost. Letting $\mu$ and $\nu \in \textnormal{Prob}(\sR^d \times \sR^d)$ be two probability measures, we define the Wasserstein-$1$ distance with a Euclidean cost function as

\begin{align*}
    \gW(\mu, \nu) = \inf_{\pi \in \Pi(\mu, \nu)}  \E_{\substack{(X,V) \\ (X', V')} \sim \pi} \left[ \left\|X-X' \right\| + \left\|V-V' \right\| \right],
\end{align*}
where $\Pi(\mu, \nu)$ denotes the set couplings whose marginals are $\mu$ and $\nu$, respectively. We are now ready to state our theorem.

\begin{theorem}
If $\lambda K > 1 - \lambda$ and $f$ projects the representation to a unit hypersphere, we have
\begin{align*}
&-\E \left[\gL_{\textnormal{RINCE}}^{\lambda, q=1}(\textbf{s})\right] \leq \frac{\Lip(f) \cdot (1 - \lambda) \cdot e^{1/t}}{t} \gW_1( P_{XV}^\phi,  P_X^\phi P_V^\phi).
\end{align*}
\label{thm_w}
\end{theorem}
\begin{proof}
By the additivity of expectation, we can bound the negative symmetric loss as follows
\begin{align*}
&-\E \left[\gL_{\textnormal{RINCE}}^{\lambda, q=1}(\textbf{s})\right] 
\\
&= \E_{\substack{x \sim P_X \\ v \sim P_{V|X=x} \\ v_i \sim P_V}} [(1 - \lambda)e^{f(\phi(x))^T f(\phi(v)/t} - \lambda \sum_{i=1}^K e^{f(\phi(x))^T f(\phi(v_i)/t} ]
\\
&= \E_{(x,v) \sim P_{XV}} \left[(1 - \lambda)e^{f(\phi(x))^T f(\phi(v)/t} \right] 
 - \E_{\substack{x \sim P_X \\ v_i \sim P_V}}\left[\lambda \sum_{i=1}^K e^{f(\phi(x))^T f(\phi(v_i)/t} \right]
\\
&= \E_{(x,v) \sim P_{XV}} \left[(1 - \lambda)e^{f(\phi(x))^T f(\phi(v)/t} \right] 
- \E_{\substack{x \sim P_X \\ v_i \sim P_V}} \left[\lambda \sum_{i=1}^K e^{f(\phi(x))^T f(\phi(v_i)/t} \right]
\\
&= \E_{(x,v) \sim P_{XV}} \left[(1 - \lambda)e^{f(\phi(x))^T f(\phi(v)/t} \right] 
- \lambda K\E_{\substack{x \sim P_X \\ v \sim P_V}} \left[ e^{f(\phi(x))^T f(\phi(v)/t} \right]
\\
&\leq (1 - \lambda) \cdot (\E_{(x,v) \sim P_{XV}} \left[e^{f(\phi(x))^T f(\phi(v)/t} \right] - \E_{\substack{x \sim P_X \\ v \sim P_V}} \left[ e^{f(\phi(x))^T f(\phi(v)/t} \right]),
\end{align*}
where the last equality follows by $\lambda K > 1 - \lambda$. Note that for $\frac{-1}{t} \leq s \leq \frac{1}{t}$, which implies $ |\nabla_s e^s| \leq e^{1/t}$. Therefore, by the mean value theorem, we have
\begin{align*}
    &|e^{f(\phi(x))^T f(\phi(v))/t} - e^{f(\phi(x'))^T f(\phi(v'))/t}|
    \\
    &\leq \frac{e^{1/t}}{t} | \langle f(\phi(x)), f(\phi(v)) \rangle - \langle f(\phi(x')), f(\phi(v')) \rangle| \tag{Mean Value Theorem}
    \\
    &= \frac{e^{1/t}}{t} | \langle f(\phi(x)) - f(\phi(x')), f(\phi(v)) \rangle + \langle f(\phi(x')), f(\phi(v) - f(\phi(v')) \rangle|
    \\
    &\leq \frac{e^{1/t}}{t}(| \langle f(\phi(x)) - f(\phi(x')), f(\phi(v)) \rangle| + |\langle f(\phi(x')), f(\phi(v) - f(\phi(v')) \rangle|)
    \\
    &\leq \frac{e^{1/t}}{t}(\|f(\phi(x)) - f(\phi(x'))\|\|f(\phi(v))\| + \|f(\phi(v) - f(\phi(v'))\| \|f(\phi(x'))\|) \tag{Cauchy–Schwarz Ineq.}
    \\
    &= \frac{e^{1/t}}{t}(\|f(\phi(x)) - f(\phi(x'))\| + \|f(\phi(v) - f(\phi(v'))\| ) \tag{$f(\phi(x))$ is unit norm}
    \\
    &\leq \frac{\Lip(f) \cdot e^{1/t}}{t}(\|\phi(x) - \phi(x')\| + \|\phi(v) - \phi(v')\|)
    \\
    &= \frac{\Lip(f) \cdot e^{1/t}}{t} d((\phi(x), \phi(v)), (\phi(x'), \phi(v'))).
\end{align*}
We can see that the Lipschitz constant of $\exp(f(\cdot, \cdot))$ with respect to the metric $d$ is bounded by $\frac{\Lip(f) \cdot e^{1/t}}{t}$. Therefore, by Kantorovich-Rubinstein duality, we have
\begin{align*}
&-\E \left[\gL_{\textnormal{RINCE}}^{\lambda, q=1}(\textbf{s})\right] 
\\
&\leq (1 - \lambda) \cdot (\E_{(x,v) \sim P_{XV}} [e^{f(\phi(x))^T f(\phi(v)/t}] 
\\
&\qquad\qquad\qquad\qquad - \E_{\substack{x \sim P_X \\ v \sim P_V}}[ e^{f(\phi(x))^T f(\phi(v)/t}]),
\\
&\leq \frac{\Lip(f) \cdot (1 - \lambda) \cdot e^{1/t}}{t} \gW_1(\phi_\# P_{XV}, \phi_\# P_X \cdot \phi_\#P_V)
\end{align*}
\end{proof}

\subsection{Noisy Wasserstein Dependency Measure}
\label{sec_noisy_wdm}
The result is a simple combination of Corollary \ref{cor_noisy} and Theorem \ref{thm_w}. If $\lambda \geq \frac{\eta K - \eta + 1}{\eta K - \eta + 1 + K}$, by the assumption of additive noisy models and the symmetry of loss, we have
\begin{align*}
&-\E \left[\gL_{\textnormal{RINCE}}^{\lambda, q=1}(\textbf{s})\right] 
\\
&= \E_{\substack{(x,v) \sim P_{XV}^\eta \\ v_i \sim P_V}} [(1 - \lambda)e^{f(\phi(x))^T f(\phi(v)/t}   - \lambda \sum_{i=1}^K e^{f(\phi(x))^T f(\phi(v_i)/t} ]
\\
&= \E_{(x,v) \sim P_{XV}^\eta} \left[(1 - \lambda)e^{f(\phi(x))^T f(\phi(v)/t} \right]  - \E_{\substack{x \sim P_X \\ v_i \sim P_V}}\left[\lambda \sum_{i=1}^K e^{f(\phi(x))^T f(\phi(v_i)/t} \right]
\\
&= (1 - \lambda)(1-\eta)\E_{(x,v) \sim P_{XV}} \left[e^{f(\phi(x))^T f(\phi(v)/t} \right]  - K \cdot (\lambda - \eta + \eta \lambda) \E_{\substack{x \sim P_X \\ v \sim P_V}}\left[  e^{f(\phi(x))^T f(\phi(v)/t} \right] \tag{symmetry}
\\
&\leq (1 - \lambda)(1 - \eta) \cdot (\E_{(x,v) \sim P_{XV}} \left[e^{f(\phi(x))^T f(\phi(v)/t} \right] - \E_{\substack{x \sim P_X \\ v_i \sim P_V}} \left[ e^{f(\phi(x))^T f(\phi(v_i)/t} \right]) \tag{$\lambda \geq \frac{\eta K - \eta + 1}{\eta K - \eta + 1 + K}$}
\\
& \leq (1-\eta) \cdot \frac{\Lip(f) \cdot (1 - \lambda) \cdot e^{1/t}}{t} \gW_1(\phi_\# P_{XV}, \phi_\# P_X \cdot \phi_\#P_V)
\\
&= (1-\eta) \cdot L \cdot I_\gW(\phi(X),\phi(V)).
\end{align*}

\subsection{InfoNCE is not symmetric}
\label{sec_infonce_asymmetric}
Note that by taking the derivative with respect to the prediction score $s$, the definition is equivalent to $\frac{\partial \gL(s, 1)}{\partial s} + \frac{\partial \gL(s, -1)}{\partial s} = 0 \;\forall s \in \sR$.

\begin{align*}
    \frac{\partial \gL_{\textnormal{InfoNCE}}(\textbf{s})}{\partial s^+} &= \frac{-1}{\gL_{\textnormal{InfoNCE}}(\textbf{s})} \cdot \frac{e^{s^+} \cdot \sum_{i=1}^K e^{s_i^-}}{(e^{s^+} + \sum_{i=1}^K e^{s_i^-})^2}
    \\
    \frac{\partial \gL_{\textnormal{InfoNCE}}(\textbf{s})}{\partial s_i^-} &= \frac{-1}{\gL_{\textnormal{InfoNCE}}(\textbf{s})} \cdot \frac{e^{s^+}(1 - e^{s_i^-}) + \sum_{i=1}^K e^{s_i^-}}{(e^{s^+} + \sum_{i=1}^K e^{s_i^-})^2}.
\end{align*}
Within a batch of data, the gradients with respect to $s^+$ and $s^-$ are entangled and do not sum to a constant, which fail to meet the symmetry condition.

\section{Experiment Details}
\label{sec_exp_detail}

\subsection{CIFAR-10}
We follow the experiment setup in \citep{chuang2020debiased}, where the SimCLR \citep{chen2020simple} models are trained with Adam optimizer for 500 epochs with learning rate 0.001 and weight decay 1e$-$6. The encoder is ResNet-50 and the dimension of the latent vector is 128. The temperature is set to $t=0.5$. The models are then evaluated by training a linear classifier for 100 epochs with learning rate 0.001 and weight decay 1e$-$6. We use the PyTorch code in Figure \ref{fig_code} to generate the data augmentation noise.

\begin{figure*}[htbp]
\begin{lstlisting}[language=Python]
def get_train_transform(noise_rate):
    train_transform = transforms.Compose([
        transforms.RandomResizedCrop(32),
        transforms.RandomApply([transforms.RandomResizedCrop(32, scale=(0.2, 0.2))], p=noise_rate),
        transforms.RandomHorizontalFlip(p=0.5),
        transforms.RandomApply([transforms.ColorJitter(0.4, 0.4, 0.4, 0.1)], p=0.8),
        transforms.RandomGrayscale(p=0.2),
        transforms.GaussianBlur(kernel_size=int(0.1*32)),
        transforms.ToTensor(),
        transforms.Normalize([0.4914, 0.4822, 0.4465], [0.2023, 0.1994, 0.2010])])
    return train_transform
\end{lstlisting}
\caption{PyTorch code for CIFAR-10 data augmentation noise.} 
\label{fig_code}
\end{figure*}

\subsection{ImageNet}
\paragraph{
SimCLR} We adopt the SimCLR implementation\footnote{\url{https://github.com/PyTorchLightning/lightning-bolts/tree/master/pl_bolts/models/self_supervised/simclr}} from PyTorch Lightning \citep{falcon2020framework}. In addition, we spot a bug and fix the implementation of negative masking of PyTorch Lightning according to Figure \ref{fig_imgnet_code} and achieve 68.9 top-1 accuracy on ImageNet (the one reported in the PyTorch Lightning's website is 68.4). To implement RINCE, we only modify the lines that calculates loss according to Figure \ref{fig_objective_code}. 

\paragraph{
Mocov3} We adopt the official code\footnote{\url{https://github.com/facebookresearch/moco-v3}} from Mocov3 \citep{chen2021empirical}. To implement RINCE, we only modify the lines that calculates loss in \texttt{moco/builder.py} according to Figure \ref{fig_objective_code}.

\begin{figure*}[htbp]
\begin{lstlisting}[language=Python]
def compute_neg_mask(self):
    total_images = self.num_nodes * self.gpus * self.batch_size * self.num_pos
    world_size = self.num_nodes * self.gpus
    batch_size = self.batch_size * self.num_pos
    orig_images = self.batch_size
    rank = int(os.environ["LOCAL_RANK"])

    neg_mask = torch.zeros(batch_size, total_images)
    all_indices = np.arange(total_images)
    pos_members = orig_images * world_size * np.arange(self.num_pos)
    for anchor in np.arange(self.num_pos):
        for img_idx in range(orig_images):
            delete_inds = orig_images * rank + img_idx + pos_members
            neg_inds = torch.tensor(np.delete(all_indices, delete_inds)).long()
            neg_mask[anchor * orig_images + img_idx, neg_inds] = 1
    neg_mask = neg_mask.cuda(non_blocking=True)

    return neg_mask

def nt_xent_loss(self, out_1, out_2, temperature):
    if torch.distributed.is_available() and torch.distributed.is_initialized():
        out_1_dist = SyncFunction.apply(out_1)
        out_2_dist = SyncFunction.apply(out_2)
    else:
        out_1_dist = out_1
        out_2_dist = out_2

    out = torch.cat([out_1, out_2], dim=0)
    out_dist = torch.cat([out_1_dist, out_2_dist], dim=0)

    similarity = torch.exp(torch.mm(out, out_dist.t()) / temperature)
    
    #################################### original code ####################################
    # # from each row, subtract e^(1/temp) to remove similarity measure for x1.x1         #
    # neg = similarity.sum(dim=-1)                                                        #
    # row_sub = Tensor(neg.shape).fill_(math.e ** (1 / temperature)).to(neg.device)       #
    # neg = torch.clamp(neg - row_sub, min=eps)  # clamp for numerical stability          #
    #######################################################################################
    
    neg_mask = self.compute_neg_mask()
    neg = torch.sum(similarity * neg_mask, 1)

    pos = torch.exp(torch.sum(out_1 * out_2, dim=-1) / temperature)
    pos = torch.cat([pos, pos], dim=0)

    loss = -(torch.mean(torch.log(pos / (pos + neg))))

    return loss
\end{lstlisting}
\caption{PyTorch Lightening implementation of SimCLR. The original implementation of negative masking (commented out) is problematic because it subtracts $e^{1/t}$ to remove similarity measure for pairs that consist of the same images. However, subtracting a constant does not alter the gradient with respect to the model parameters. In particular, there are still gradients backpropagating through the false positive pairs. We fix it by directly filtering out those false pairs with a negative mask.} 
\label{fig_imgnet_code}
\end{figure*}

\subsection{Kinetics-400}
We adopt the official implementation\footnote{\url{https://github.com/facebookresearch/AVID-CMA}} from \citep{morgado2021audio}. Similarly, we only modify the loss function in the \texttt{criterions} directory. In particular, we use the SimCLR style implementation for both InfoNCE and RINCE loss. We also adopt the same hyperparameters described in the git repository for training. We set the learning rate to $1e-3$ to finetune the models on downstream classification tasks such as UCF101 and HMDB51 with the provided evaluation code.

\subsection{ACAV100M}
We again modify the official implementation of \citep{morgado2021audio} for the ACAV100M experiments, where we modify the data loader to adopt it to ACAV100M. Different from Kinetics-400 experiments, the input size is set to $8 \times 224^2$ during the finetuning process for computational efficiency. We again use the exact same set of hyperparameters from \citep{morgado2021audio} for both training and testing.

\subsection{TU-Dataset}
We adopt the official implementation\footnote{\url{https://github.com/Shen-Lab/GraphCL/tree/master/unsupervised_TU}} from \citep{you2020graph}. To implement RINCE, we only modify the loss in \texttt{gsimclr.py} file.

\section{Additional Results}
\subsection{Exact Number of CIFAR-10 and ACAV100M Experiments}
We first provide the exact numbers for CIFAR-10 and ACAV100M experiments.
\begin{table}[htp]
     \centering
     \footnotesize
    \begin{tabularx}{0.45\textwidth}{r|*{6}{c}}
    \toprule
    \bf $\eta$ &  InfoNCE &  $q  = 0.01$ &  $q  = 0.1$ & $q  = 0.5$ & $q  = 1.0$ &   \\
    \hline
    \hline
   0.0 & 93.4$\pm$0.2  & 93.4$\pm$0.2 & 93.2$\pm$0.1 & 93.3$\pm$0.1 & 93.0$\pm$0.2   \\
    0.2 & 93.1$\pm$0.1 & 93.3$\pm$0.3 & 93.0$\pm$0.1 &  93.2$\pm$0.2 & 92.9$\pm$0.3  \\
    0.4 & 90.7$\pm$0.2 & 93.0$\pm$0.2 & 92.0$\pm$0.9 & 93.1$\pm$0.1 & 92.8$\pm$0.1 \\
    0.6 & 88.2$\pm$0.4& 90.8$\pm$0.2 & 90.6$\pm$0.3 & 92.9$\pm$0.2  & 92.4$\pm$0.2   \\
    0.8 & 87.1$\pm$0.5 & 89.1$\pm$0.2 & 89.3$\pm$0.1  & 89.9$\pm$0.3 & 91.6$\pm$0.3 \\
    1.0 & 87.1$\pm$1.0 & 88.7$\pm$0.1 & 89.3$\pm$0.4 & 89.3$\pm$0.6 & 88.2$\pm$0.3\\
    \bottomrule
        \end{tabularx}
    \vspace{-1mm}
    \caption{CIFAR-10 Label Noise}
\end{table}

\begin{table}[htp]
     \centering
     \footnotesize
    \begin{tabularx}{0.45\textwidth}{r|*{6}{c}}
    \toprule
    \bf $\eta$ &  InfoNCE &  $q  = 0.01$ &  $q  = 0.1$ & $q  = 0.5$ & $q  = 1.0$ &   \\
    \hline
    \hline
   0.0 & 91.1$\pm$0.1  & 91.6$\pm$0.1 & 91.5$\pm$0.1 & 91.8$\pm$0.2 & 90.7$\pm$0.1   \\
    0.2 & 89.3$\pm$0.1 & 89.8$\pm$0.2 & 89.7$\pm$0.1 &  90.4$\pm$0.1 & 90.9$\pm$0.1  \\
    0.4 & 87.3$\pm$0.4 & 87.7$\pm$0.5 & 87.5$\pm$0.2 & 88.8$\pm$0.1 & 89.0$\pm$0.1 \\
    0.6 & 84.5$\pm$0.2& 85.4$\pm$0.2 & 85.3$\pm$0.2 & 86.6$\pm$0.1  & 86.3$\pm$0.2   \\
    0.8 & 80.6$\pm$0.1 & 81.2$\pm$0.2 & 80.3$\pm$0.2  & 82.5$\pm$0.2 & 82.8$\pm$0.3 \\
    1.0 & 71.0$\pm$0.5 & 71.2$\pm$0.6 & 71.8$\pm$0.4 & 71.5$\pm$0.3 & 72.7$\pm$0.2\\
    \bottomrule
        \end{tabularx}
    \vspace{-1mm}
    \caption{CIFAR-10 Augmentation Noise}
\end{table}

\begin{table}[htp]
     \centering
     \footnotesize
    \begin{tabularx}{0.48\textwidth}{r|*{6}{c}}
    \toprule
    model &  20K &  50K &  100K & 200K & 500K &   \\
    \hline
    \hline
   InfoNCE (100 epoch) & 72.482  & 75.205 & 77.161 & 79.937 & 82.717   \\
    InfoNCE (150 epoch) & 72.429 & 76.13 & 78.8 &  80.095 & 83.082  \\
    InfoNCE (200 epoch) & 72.429 & 76.183 & 78.641 & 79.94 & 83.388 \\
    RINCE (100 epoch) & 73.635 & 76.685 &78.694 & 81.153 & 83.505   \\
    RINCE (150 epoch) & 74.632 & 77.505 & 79.064  & 82.263 & 83.399 \\
    RINCE (200 epoch) & 74.253 & 78.086 & 79.355 & 82.368 & 83.769\\
    \bottomrule
        \end{tabularx}
    \vspace{-1mm}
    \caption{Top1 accuracy on UCF101 of models trained on ACAV100M.}
\end{table}

\subsection{Positive Scores and Views, Continue}
We extend our analysis of Figure \ref{fig_imgnet} to InfoNCE baseline and discuss the impact of implicit weighting. We can see that the positive scores in both InfoNCE and RINCE models are correlated to the noisiness of positive pairs.
\begin{figure}[!htb]
\begin{center}   \includegraphics[width=0.85\linewidth]{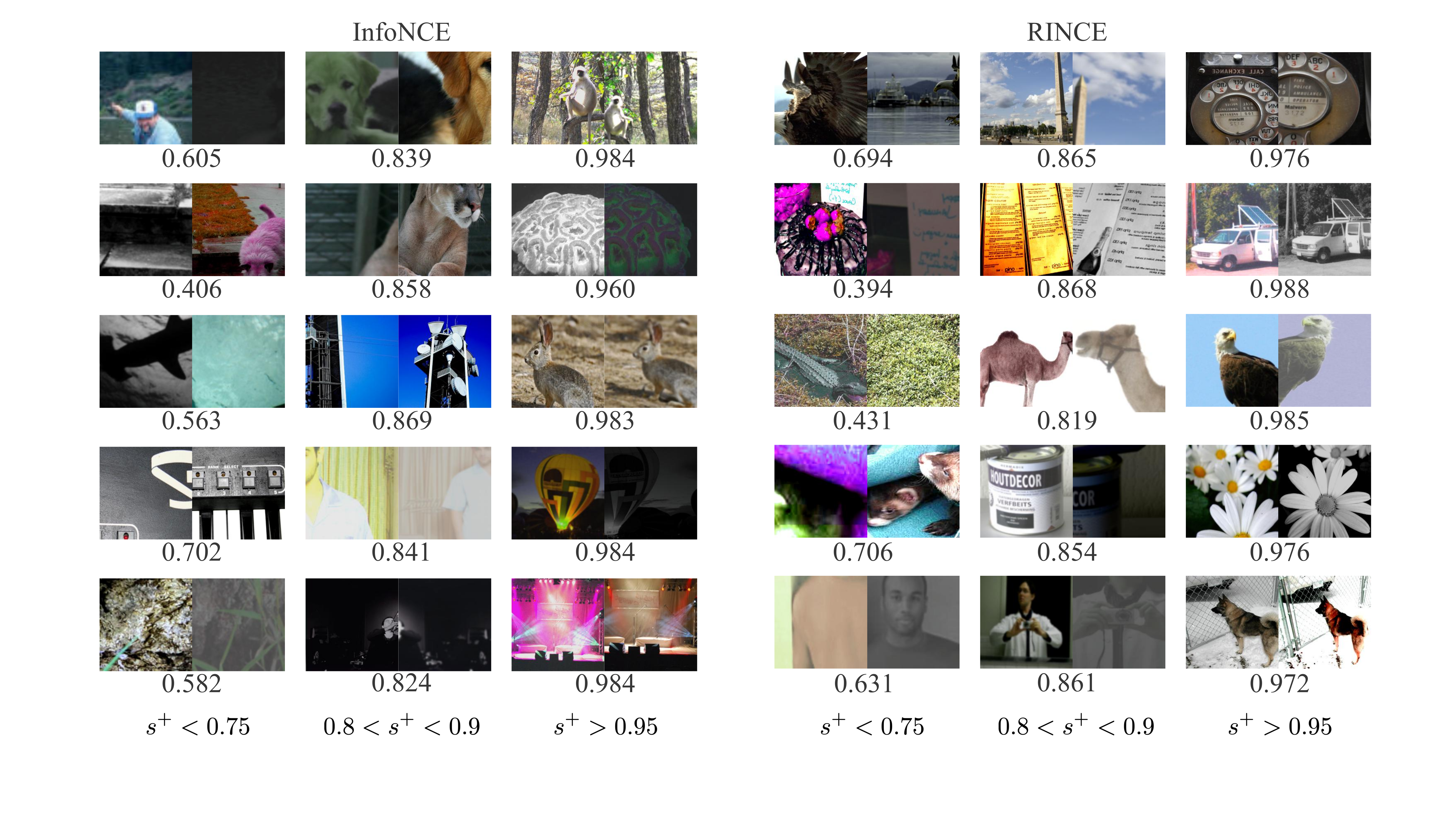}
\end{center}
   \caption{\textbf{Positive pairs and their scores.} The corresponding positive scores are shown below the image pairs. The positive scores $s^+ \in [-1, 1]$ are output by the trained InfoNCE and RINCE model (temperature $=$ 1). Pairs that have lower scores are visually noisy, while informative pairs often have higher scores.}
\label{fig_imgnet_appendix_1}

\end{figure}

We then study the distribution of positive scores and compare the positive scores output by InfoNCE and RINCE on noisy views. As Figure \ref{fig_imgnet_appendix_2} (a) shows, the positive scores of clean pairs output by RINCE is slightly higher, making the density of RINCE around score 1.0 larger than InfoNCE. Figure \ref{fig_imgnet_appendix_2} (b) gives a closer look on scores versus noisy views. We can see that InfoNCE tends to output higher scores for noisy views than RINCE, corroborating our analysis: InfoNCE tends to maximize the positive score of hard (noisy) pairs. This inherently makes the positive scores of clean pairs lower for InfoNCE, explaining the discrepancy between InfoNCE and RINCE in (a).

\begin{figure}[!htb]
\begin{center}   \includegraphics[width=0.7\linewidth]{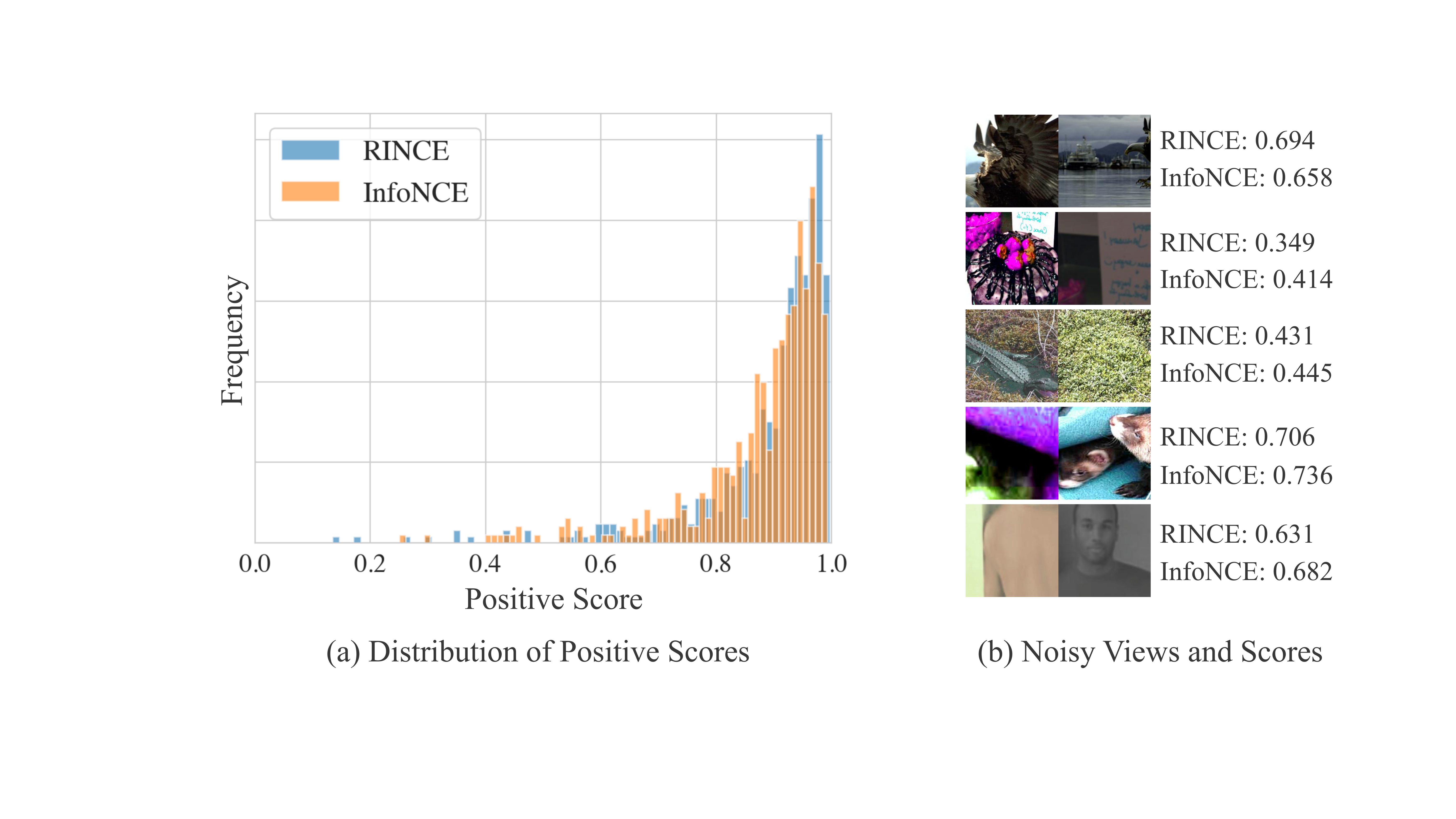}
\end{center}
   \caption{\textbf{Comparison between RINCE and InfoNCE.} (a) Distribution of Positive Scores for RINCE and InfoNCE; (b) InfoNCE outputs higher scores for noisy pairs. }
\label{fig_imgnet_appendix_2}
\end{figure}

\subsection{Ablation Study on $\lambda$}
Finally, we provide an ablation study on how $\lambda$ affect the performance of RINCE with CIFAR-10 augmentation noise experiments. We can see that in both clean and noise setting, RINCE is not sensitive to the choice of $\lambda$ as long as it is not too large. Therefore, we simply set $\lambda = 0.01$ for all vision experiments and $\lambda = 0.025$ for graph experiments.
\begin{table}[htp]
     \centering
     \footnotesize
    \begin{tabularx}{0.27\textwidth}{l|*{2}{c}}
            \toprule
            Noise Rate & 0.0 & 0.4   \\
            \hline
            \hline
            RINCE ($\lambda=0.01$) & 91.54 & 89.65
            \\
            RINCE ($\lambda=0.05$) & \textbf{91.81} & \textbf{89.81}
            \\
            RINCE ($\lambda=0.1$) & 91.32 & 89.9
            \\
            RINCE ($\lambda=0.2$) & 90.55 & 89.69
            \\
            RINCE ($\lambda=0.4$) & 90.89 & 89.39
            \\
            \bottomrule
        \end{tabularx}
    \vspace{-1mm}
    \caption{CIFAR-10 Augmentation Noise}
\end{table}

\end{document}